\documentclass[letterpaper,times]{IONconf}
\usepackage{amsmath}
\usepackage{amssymb}
\usepackage{bm}
\usepackage{graphicx}
\usepackage{caption}
\usepackage{subcaption}
\usepackage{algorithmic}
\usepackage{array}
\usepackage{url}
\usepackage{natbib}
\usepackage[hidelinks]{hyperref}
\usepackage{subcaption}
\usepackage{todonotes}
\usepackage{here}
\usepackage{algorithm2e}
\usepackage{siunitx}

\usepackage{multicol}
\usepackage{multirow}
\usepackage{lscape}
\usepackage{array}
\usepackage{color}

\newcommand{\R}{\mathbb{R}}

\newcommand{\N}{\mathbb{N}}

\newcommand{\matrank}{\mathrm{rank}}
\newcommand{\diag}{\mathrm{diag}}
\newcommand{\onesvec}{\mathbf{1}}
\newcommand{\zeromat}{\mathbf{0}}
\newcommand{\eye}{\mathbf{I}}
\newcommand{\hadamard}{\circ}

\newcommand{\satrange}{r}
\newcommand{\singlepoint}{\bm{x}}
\newcommand{\fault}{f}
\newcommand{\faultmag}{\bar{f}}
\newcommand{\noise}{w}
\newcommand{\sqrtterm}{s}

\newcommand{\rrefleftover}{\Delta}

\newcommand{\PointMat}{\mathbf{X}}
\newcommand{\GCMat}{\mathbf{J}}

\newcommand{\EDM}{\mathbf{D}}
\newcommand{\EDMnoisy}{\Tilde{\mathbf{D}}}

\newcommand{\GCEDM}{\mathbf{G}}
\newcommand{\GCEDMnoisy}{\Tilde{\mathbf{G}}}

\newcommand{\FaultMat}{\mathbf{F}}

\newcommand{\SqrtMat}{\mathbf{S}}

\newcommand{\FaultMatSuper}{\FaultMat^{n, m}}
\newcommand{\SqrtMatSuper}{\SqrtMat^{n, d}}

\newcommand{\meanPoints}{\mu}

\newcommand{\teststatistic}{\gamma_{\text{test}}}
\newcommand{\teststatisticthreshold}{\bar{\gamma}_{\text{test}}}

\usepackage{amsthm}
\usepackage[english]{babel}
\newtheorem{proposition}{Proposition}[section]
\newtheorem{corollary}{Corollary}[proposition]

\newcommand{\bhline}[1]{\noalign{\hrule height #1}}   
\newcolumntype{I}{!{\vrule width 1.5pt}}

\title{Autonomous Constellation Fault Monitoring with Inter-satellite Links: A Rigidity-Based Approach}

\author{
    Keidai~Iiyama, Daniel Neamati, and Grace~Gao \\%
    \vspace{1mm}
    \textit{Stanford~University}
    }

\begin{document}

\maketitle

\section*{biography}


\biography{Keidai Iiyama}{is a Ph.D. candidate in the Department of Aeronautics and Astronautics at Stanford University. He received his M.E. degree in Aerospace Engineering in 2021 from the University of Tokyo, where he also received his B.E. in 2019. His research interests include positioning, navigation, and timing of spacecraft and planetary rovers. }

\biography{Daniel Neamati}{is a Ph.D. candidate in the Department of Aeronautics and Astronautics at Stanford University. He received his bachelor's degree in Mechanical Engineering, with a minor in Planetary Science, from the California Institute of Technology. His research interests include GNSS, geospatial information, autonomous decision-making, and risk-aware localization.}

\biography{Grace Gao}{is an associate professor in the Department of Aeronautics and Astronautics at Stanford University.
Before joining Stanford University, she was an assistant professor at University of Illinois at Urbana-Champaign.
She obtained her Ph.D. degree at Stanford University.
Her research is on robust and secure positioning, navigation, and timing with applications to manned and unmanned aerial vehicles, autonomous driving cars, as well as space robotics.}

\section*{Abstract}
To address the need for robust positioning, navigation, and timing services in lunar environments, this paper proposes a novel fault detection framework for satellite constellations using inter-satellite ranging (ISR). 
Traditionally, navigation satellites can depend on a robust network of ground-based stations for fault monitoring. 
However, due to cost constraints, a comprehensive ground segment on the lunar surface is impractical for lunar constellations.
Our approach leverages vertex redundantly rigid graphs to detect faults without relying on precise ephemeris. 
We model satellite constellations as graphs where satellites are vertices and inter-satellite links are edges. 
We identify faults through the singular values of the geometric-centered Euclidean distance matrix (GCEDM) of 2-vertex redundantly rigid sub-graphs. 
The proposed method is validated through simulations of constellations around the Moon, demonstrating its effectiveness in various configurations. 
This research contributes to the reliable operation of satellite constellations for future lunar exploration missions.
\section{Introduction}
To meet the growing need for robust positioning, navigation, and timing (PNT) services at the lunar surface and lunar orbits, NASA and its international partners are collaborating to develop LunaNet~\citep{Israel2020}, a network of networks providing data relay, PNT, detection, and science services.
In LunaNet, these services are provided through cooperation among interoperable systems that evolve over time to meet the growing needs for these services, efficiently establishing a reliable, sustainable, and scalable network. 

It is crucial to monitor LunaNet navigation satellites for the reliable operation of safety-critical missions. 
The quality of lunar navigation signals can be compromised by various system faults, such as clock runoffs (e.g. phase and frequency jumps~\citep{weiss2010board}), unflagged maneuvers, failures in satellite payload signal generation components, and code-carrier incoherence~\citep{Walter2018}.
For terrestrial GNSS, the system fault alerts are provided as integrity information from GNSS augmentation systems, such as the satellite-based augmentation system (SBAS)~\citep{van2009gps}. 
There also exist algorithms to monitor and remove faults on the receiver side, known as receiver autonomous integrity monitoring (RAIM)~\citep{MISRA_1993}.
However, given the limited resources of lunar missions, it is preferable to perform fault monitoring within the navigation system. 
While the LunaNet Relay Service Documents (SRD) state that robustness of the signal should be a key consideration for LunaNet~\citep{nasa2022srd}, the specific methodology to monitor faults on LunaNet satellites is yet solidified.

Satellite fault monitoring in LunaNet can be challenging because there are no dedicated monitoring station networks on the Moon, especially in the early stage of operation. 
On Earth, SBAS monitors satellite faults by collecting GNSS signals at monitoring stations that are accurately located, and processes them in a central computing center where differential correction and integrity messages are calculated~\citep{van2009gps}. 
However, navigation satellite systems on the Moon or Mars will likely have none or very few monitoring stations on the planet's surface to monitor the signals, due to stringent cost constraints to deploy and maintain these stations. 
Therefore, onboard algorithms that can monitor satellite faults are desired for LunaNet.

One promising approach for autonomous satellite fault monitoring is to use inter-satellite ranging (ISR), a concept that has been investigated for terrestrial GNSS constellations~\citep{wolf_onboard_2000, rodriguez-perez_inter-satellite_2011}. 
The proposed autonomous satellite fault monitoring algorithms from previous works assume the availability of precise ephemeris to compute the estimated range between the satellites, which is used to compute the residuals over-bounded by Gaussian distributions. 
However, obtaining a precise ephemeris is challenging for future lunar and Martian constellations due to the limited number of monitoring stations and the lower stability of onboard clocks due to stringent cost requirements.
Moreover, if the ISR measurements are used for orbit determination and time synchronization (ODTS) to generate the ephemeris of the navigation satellites as proposed in the LunaNet concept, this would create a chicken-and-egg problem of ODTS and satellite fault detection.
The previously proposed algorithms also assume that the ISR measurements are sufficiently precise, well-calibrated, and contain no faults, which may be too strong of an assumption in practice. 

To tackle these problems, we propose a satellite fault detection framework that uses two-way ISR measurements and does not rely on ephemeris information.
In particular, out of the four faults previously mentioned, we 
target the clock frequency jump, which generates bias on the two-way ISR measurements. 
While two-way ISR measurements can cancel the time synchronization error between the two satellites, it is sensitive to the frequency difference of the two clocks \citep{Akawieh2016, rathje2024time}.

Our algorithm uses the ISR bias generated from faults to detect a deformation in the 3D rigid graph.
We start by modeling satellite constellations with ISR measurements as a graph, where satellites are modeled as vertices and links between satellites are modeled as edges with weights as their measured ranges. 
When the constellation size is sufficiently large and a meshed network within the constellation can be constructed, this graph contains multiple subgraphs of $k$ satellites that are fully connected, which are called $k$-cliques. 
$k$-cliques of $k \geq 5$ are known to be 2-vertex rigid in 3-dimensional space, which means they remain rigid (graphs cannot be susceptible to continuous flexing are called rigid graphs), after removing any vertex from the graph~\citep{alireza_motevallian_robustness_2015}.
We show that if a graph is 2-vertex rigid in 3-dimensional space, we can detect a fault satellite by checking if the given set of ranges is realizable in 3-dimensional space. 
Therefore, we can determine if there is a fault satellite within the nodes by checking if the given set of ranges is realizable in 3-dimensional space, for each of the $k$-clique subgraphs with $k \geq 5$.

Our proposed method monitors the singular values of the geometric-centered Euclidean Distance Matrix (GCEDM) \citep{dokmanic2015euclidean} constructed from the range measurements, to check if a given set of ranges are realizable in 3-dimensional space. 
This is inspired by the GNSS fault detection algorithm by \citep{Knowles2023}, where they propose to monitor the 4th and 5th singular value of the GCEDM, constructed from the observed range between the user and the GNSS satellite and the ephemeris of the GNSS satellites to identify GNSS satellites with faults. 
Using the fact that the 4th and 5th singular values increase when a fault is present, their method identifies the presence of a fault satellite whenever the sum of these two singular values goes over a certain threshold. In this paper, we prove why the 4th and 5th singular value increases under the presence of fault, and analyze the distribution of the singular values under the presence of noise in the range measurement.
Using GCEDM is promising for fault detection because it does not require solving for the user position to identify graph realizability (and faults) and resolves the aforementioned chicken-and-egg problem.



The contribution of this paper is summarized below. This work is based on our prior conference paper presented at the 2024 ION GNSS+~\citep{Iiyama2024Rigid}.
\begin{itemize}
    \item We propose a rigidity-based online fault detection framework for satellite constellations that do not require precise ephemeris or observations at the monitoring stations.
    \item We identify the required graph topology to identify fault satellites from a set of measured inter-satellite ranges. In particular, we show that the graph has to be 2-vertex redundantly rigid to detect fault satellites.
    \item We prove several key properties about the ranks of EDMs and GCEDMs to provide mathematical backing for fault detection using the fourth and fifth singular values of the GCEDM. 
    \item We validate our algorithm in a simulated constellation around the Moon. We show how the hyper-parameters and fault magnitudes affect the fault detection performance.
\end{itemize}

The paper is arranged as follows. Section \ref{sec:topology} discuss the required graph topology to detect satellite faults from observed ranges.
Section \ref{sec:edm_property} introduces the notations and proves some properties of the rank of the EDMs and GCEDMs. In section \ref{sec:algorithm}, the fault detection framework is proposed. In section \ref{sec:simulation}, performance evaluation for lunar constellations are provided. The paper concludes in section \ref{sec:conclusion}.

\section{Required Graph Topology for Fault Detection}
\label{sec:topology}
In this section, we identify the required graph topology for fault detection.

\subsection{Definitions}
\subsubsection{Range Measurements and Satellite Faults}
Consider a graph of $n$ nodes, which corresponds to satellites. 
Let $\singlepoint_i \in \R^d,  (i \in \{1, \ldots, n\})$ be the position of the $i$th node, and $\PointMat \in \R^{d \times n}$ be a matrix collecting these points. 
In this paper, we are interested in the 3D case, $d=3$.
By establishing two-way inter-satellite links, we obtain the range measurement $r_{ij}$ between the two satellites $i$ and $j$, 
\begin{equation}
    \satrange_{ij} = \begin{cases}
        \| \singlepoint_i - \singlepoint_j\| + \noise_{ij} + \fault_{ij} = \frac{c}{2} \tau_{ij} &  (i \neq j)  \\
        0   & (i = j)  
    \end{cases}
\end{equation}
Above, $w_{ij}$ is the range measurement noise,  $\fault_{ij}$ is the bias of the range measurements, $c$ is the light speed, and $\tau_{ij}$ is the (measured) two-way travel time.
The noise $w_{ij}$ includes thermal noise in the receiver, un-calibrated instrumental delay, and light-time delay estimation errors.
We assume the measurement noise $w_{ij}$ is sampled from a Gaussian distribution $\mathcal{N}(0, \sigma_w)$.
The bias term is determined as follows
\begin{align}
    \fault_{ij} &= \fault_i + \fault_j,  \\
    \fault_k &= \begin{cases}
    \faultmag  & \text{satellite k is fault}  \\
    0  & \text{satellite k is normal}
    \end{cases}, \quad k \in \{i, j\}
\end{align}
which means we assume that the bias is zero when no satellites are in fault status. In other words, we assume a constant bias $\faultmag$ gets inserted if either of the satellites in the edges is faulty.


\subsubsection{Weighed Graph and Embeddability}
Based on the graph and observed ranges, we define a weighted graph $G = \langle V, E, W \rangle$. A \textit{realization (embedding)} of $G$ is d-dimensional Euclidean distance space, $\R^d$, is defined as a function $f_r$ that maps $V$ into $\R^d$, such that for each edge $e=(v_i, v_j) \in E$, $W(e) = r_{ij}$ \citep{jian_beyond_2010}. 

A weighted graph $G = \langle V, E, W \rangle$ is called \textit{d-embeddable} when there is a realization (embedding) $f: V \rightarrow \R^{d} $ that maps the edges to points ($\singlepoint_1, \ldots, \singlepoint_n$) in $\R^d$ space.

A weighted graph is \textit{fault disprovable} if and only if the embeddability of $G$ implies that it contains no fault satellites.

\subsubsection{Rigid and vertex redundantly rigid graphs}
A graph is called \textit{rigid} if it has no continuous deformation other than rotation, translation, and reflection while preserving the relative distance constraints between the vertices~\citep{laman_graphs_1970}. Otherwise, it is called a \textit{flexible} graph. A graph is called \textit{k-vertex redundantly rigid} if it remains rigid after removing any $(k-1)$ vertices. Similarly, if it remains rigid after removing any $(k-1)$ edges, it is called \textit{k-edge redundantly rigid}.

\begin{figure}[h]
\centering
\begin{minipage}[b]{0.33\columnwidth}
    \centering
    \includegraphics[width=0.9\columnwidth]{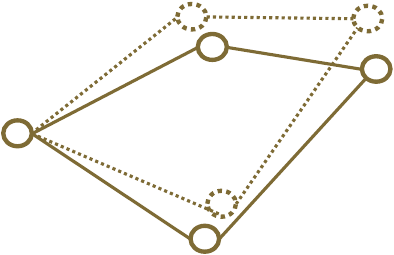}
    \subcaption{Flexible Graph}
    \label{fig:flexible_graph}
\end{minipage}
\begin{minipage}[b]{0.33\columnwidth}
    \centering
    \includegraphics[width=0.9\columnwidth]{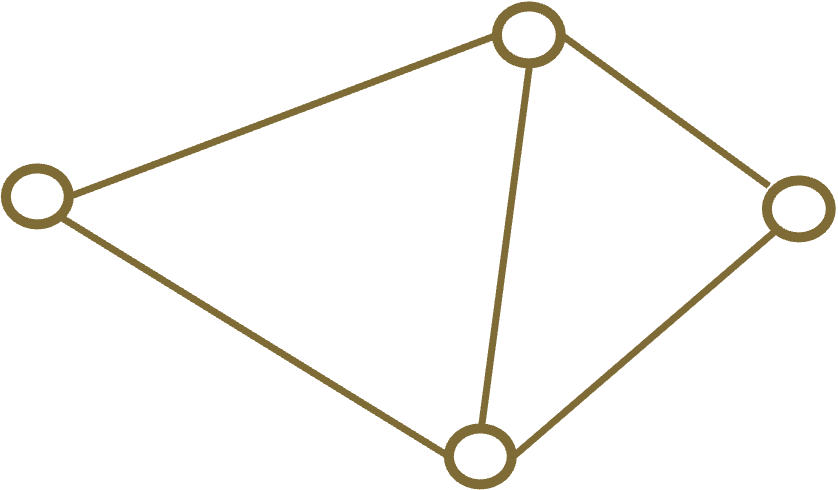}
    \subcaption{Rigid Graph}
    \label{fig:rigid_graphs}
\end{minipage}
\begin{minipage}[b]{0.33\columnwidth}
    \centering
    \includegraphics[width=0.9\columnwidth]{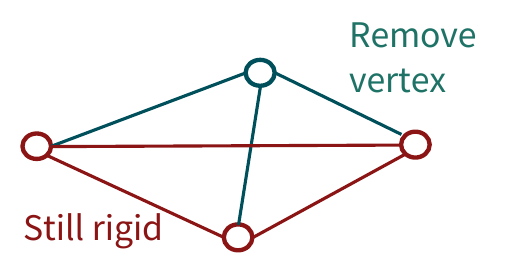}
    \subcaption{2-vertex redundantly rigid graph}
    \label{fig:red_rigid_graph}
\end{minipage}
\caption{Flexible, Rigid, and 2-vertex redundantly rigid graphs in $\R^2$}
\end{figure}

\subsection{Embeddability of Graphs with Fault Satellites}
In this section, we prove the sufficient and necessary condition for the graph to be fault disprovable in $\R^{d}$. 
In the following proofs, we assume that there is no noise $w_{ij}$ in the range measurements, and $\bar{f}$ is an arbitrary continuous random variable.

\subsubsection{Sufficient condition for fault disprovability}
\begin{proposition}
    \label{prop:fd_redudantly_rigd}
    Suppose there is a single fault satellite on a graph $G' = \langle V, E, W \rangle $ that is k-vertex ($k \geq 2$) redundantly rigid in $\R^{d}$. 
    If $G$ is $d$-embeddable, then $G$ contains no fault vertex (satellite) with probability 1.
\end{proposition}

\begin{proof}
    Assume $G$ is k-vertex ($k \geq 2$) redundantly rigid, $d$-embeddable, and contains a fault satellite. Let $v$ a fault satellite in $G$ and $u \neq v$ be another (normal) satellite in $G$. Because $G$ is $k$-vertex ($k \geq 2$) redundantly rigid, after removing $v$, it is still rigid, which means there are finite discontinuous realizations up to congruence for $G^v = G \setminus {v}$. Similarly, there are finite discontinuous realizations up to congruence for $G^u = G \setminus {u}$. In each realization of $G^v$ and $G^u$, the distance between node $v$ and node $u$ is fixed. So, in total, there are finite discrete values, denoted by set $S$, for the distance between $u$ and $v$. As $G$ is embeddable and $e_{uv} = (u, v) \in E$, the measured distance of $e_{uv}$ is in $S$. However, since $e_{uv}$ is an edge connected to a fault satellite, the measured distance of $e_{uv}$ is a continuous random variable, so with a probability of 0, $W(e) \in S$, since $S$ has measure 0. So with a probability of 1, there is no fault satellite in $G$.
\end{proof}

The visual description of the proof is shown in Figure \ref{fig:red_rigid}.

\begin{figure}[ht!]
    \centering
    \includegraphics[width=0.45\linewidth]{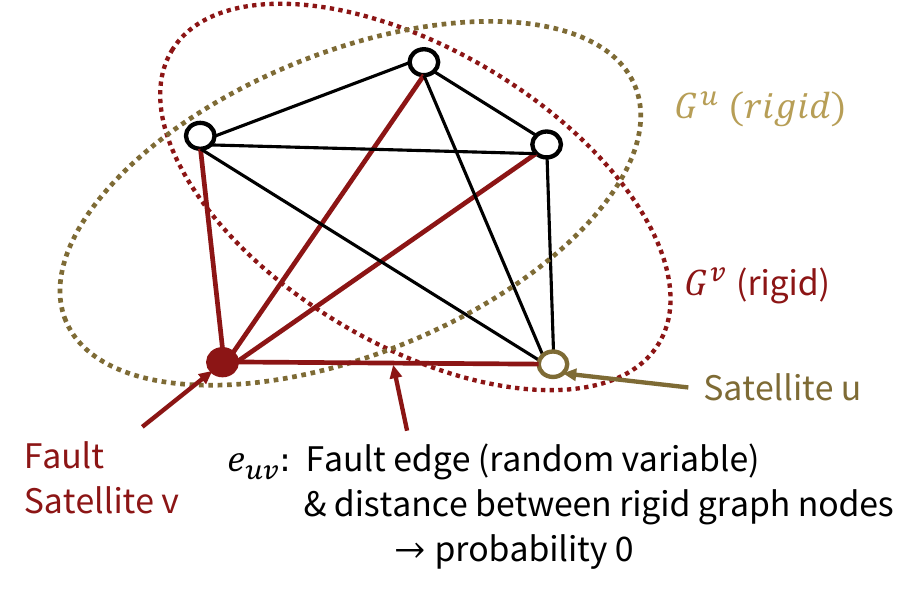}
    \caption{A visual description of the proof of Proposition \ref{prop:fd_redudantly_rigd} for $d=2$. }
    \label{fig:red_rigid}
\end{figure}

\begin{proposition}
    \label{prop:fd_red_rigid2}
    Given a weighted graph $G$, if $G$ is $k$-vetex redundantly rigid ($k \geq 2)$, then $G$ is fault disprovable.
\end{proposition}

\begin{proof}
    We can also prove that for arbitrary graphs, if $G$ is not embeddable, then $G$ contains a fault satellite. Therefore, with \ref{prop:fd_redudantly_rigd}, if $G$ is $k$-vetex ($k \geq 2)$ redundantly rigid, then $G$ is fault disprovable.
\end{proof}

\subsubsection{Necessary condition for fault disprovability}
Next, we prove the necessary condition.
\begin{proposition}
    \label{prop:fd_not_redudantly_rigd}
    Given a weighted graph $G$, if $G$ is fault disprovable, then $G$ is $k$-edge redundantly rigid ($k \geq 2)$.
\end{proposition}

\begin{proof}
    We will prove the contraposition of the proposition: \textit{If $G$ is not $k$-edge redundantly rigid ($k \geq 2)$, $G$ is not fault disprovable}.
    
    Assume $G$ is not  $k$-edge redundantly rigid and contains fault satellite. Then it is also not $k$-vertex redundantly rigid, because removing $k$ vertex results in removing at least $k$ edges.
    Therefore, there exists a vertex $v \in V$ where $G^v = G \setminus \{v\} $ is a flexible graph.
    Let U = \{$u_1, \cdots, u_m$\} a set of all nodes where $u_i \in G^v$ and $e_i = (u_i, v) \in E$. 
    In addition, since $G$ is not $k$-edge ($ k \geq 2$) redundantly rigid, we can select $v$ and $u_m$ so that subgraph $G^{e_m} = G \setminus \{(u_m, v)\}$ is flexible. In summary, we consider a (fault satellite) $v$ and another vertex $e_m$ that satisfies
    \begin{itemize}
        \item (A) $G^v = G \setminus \{v\} $ is flexible, and
        \item (B) $G^{e_m} = G \setminus \{(u_m, v)\}$ is flexible.
    \end{itemize}
    We prove that $G$ is embeddable by induction.
    
    i) First, we prove $G^v \cup e_1$ is embeddable, almost surely
    
    \textit{Proof}: It is proved that for almost all realizations of the flexible graph, the flexing space of $f_r$ contains a submanifold that is diffeomorphic to the circle~\citep{Hendrickson1992}.
    Therefore, since $G^v$ is flexible from assumption (A), the distance between node $u_1$ and $v$ will be a multivariate function for almost every point on this circle. 
    Hence, set of distances between $u_1$ and $v$ corresponding to $f_r$, denoted by $S_{d(u_i, v)}(r)$ has measure non-zero. 
    Thus, with probability greater than 0, $W(e_1) \in S_{d(u_1, v)}$, meaning $G^v \cup e_1$ is embeddable.

    ii) Next, we prove that if $\tilde{G}^i = G^v \cup e_1 \cup \cdots \cup e_i \ (1 \leq i \leq m-1)$ is embeddable, $\tilde{G}^{i+1} = G^v \cup e_1 \cup \ldots \cup e_i \cup e_{i+1}$ is embeddable, almost surely
    
    \textit{Proof}:  $\tilde{G}^{i}$ is flexible since it is a subgraph of a flexible graph $G^{e_m} = G \setminus \{(u_m, v)\}$ (from condition (B)). 
    Therefore, with probability greater than 0, $W(e_{i+1}) \in S_{d(u_{i+1}, v)}$.
    Hence, given $\tilde{G}^i$ is embeddable, $\tilde{G}^{i+1} = \tilde{G}^i \cup {e_{i+1}}$ is also embeddable.

From i) and ii), $G = \tilde{G}^m$ is embeddable. Therefore (non k-vertex rigid graph) $G$ contains fault and is embeddable, meaning it is not fault disprovable.
\end{proof}

The visual description of the proof for $\R^2$ is shown in Figure \ref{fig:not_red_rigid}.

\begin{figure}[ht!]
    \centering
    \includegraphics[width=0.8\linewidth]{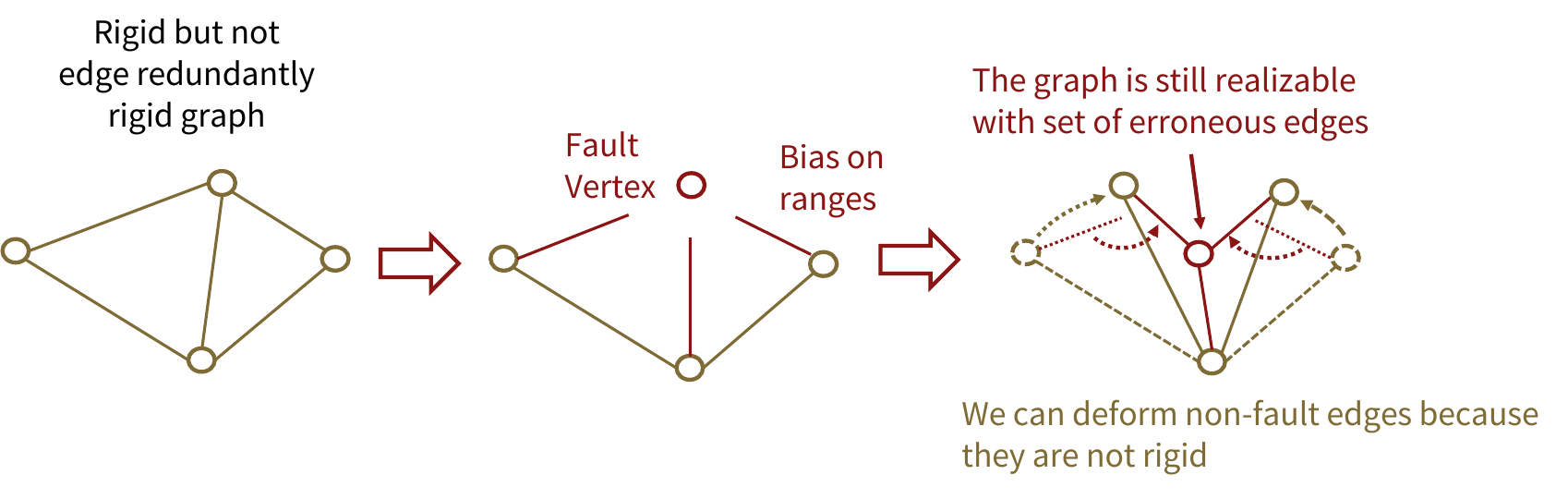}
    \caption{A visual description of the proof of Proposition \ref{prop:fd_not_redudantly_rigd} for $d=2$. We cannot detect faults by looking at the embeddability of graphs that are not 2-edge redundantly rigid. }
    \label{fig:not_red_rigid}
\end{figure}

\subsubsection{Summary}
From proposition \ref{prop:fd_red_rigid2} and \ref{prop:fd_not_redudantly_rigd}, we obtain the following proposition.

\begin{proposition}
    Given a weighted graph $G$, $G$ is fault disprovable if $G$ is k-vertex redundantly rigid ($k \geq 2)$, and only if $G$ is k-edge redundantly rigid ($k \geq 2)$.
\end{proposition}

\section{Properties of the Euclidean Distance Matrix}
\label{sec:edm_property}
In the previous section, we proved that by using vertex redundantly rigid graphs, we can detect faults by analyzing their embeddability.
Analyzing the embeddability of an arbitrary graph is known to be NP-hard ~\citep{Hendrickson1992}, but if the graph is fully connected, we can analyze them by analyzing the ranks of the GCEDM. Fully connected graphs of more than 5 nodes are 2-vertex redundantly rigid in $\R^3$.

In this section, we introduce and prove several properties of the EDM and GCEDM constructed from the observed ranges between the satellites. 
We show how we can use the singular values of the GCEDM to evaluate if a graph is embeddable in $\R^3$.
\subsection{Geometric Centered Euclidean Distance Matrix}

Consider a fully connected graph of $n$ nodes.
Using the set of observed ranges $\satrange_{ij}$, we construct an EDM, where its elements are equivalent to the square of the observed ranges ($\satrange_{ij}^2)$. Let  $\EDM^{n, d, m} \in \R^{n \times n}$ be an EDM without measurement noise ($\noise_{ij} = 0$) and $m$ nodes out of $n$ nodes being fault.
For example, a noiseless EDM with 6 satellites (in 3D space) with 2 fault satellites is denoted as $\EDM^{6, 3, 2}$.
Similarly, we define $\EDMnoisy^{n,d,m} \in \R^{n \times n}$ as EDM with measurement noise and $m$ nodes out of $n$ nodes having a fault.
A GCEDM $\GCEDM^{n,d,m}$ (or $\GCEDMnoisy^{n,d, m}$ for noisy EDM $\EDMnoisy^{n, d, m}$) is constructed from the EDM from the following operation
\begin{equation}
    \GCEDM^{n,d,m} = -\frac{1}{2} \GCMat^n \EDM^{n,d,m} \GCMat^n
\end{equation}
where $\GCMat^n$ is the geometric centering matrix as follows.
\begin{equation}
    \GCMat^{n} = \eye^n - \frac{1}{n} \onesvec {\onesvec}^{\top}
\end{equation}
where $\onesvec \in \R^n$ is the ones vector.
When no fault or noise is present $(m=0)$, the $\GCEDM^{n,d,0}$ will be positive semi-definite, and its rank will satisfy $\matrank(\GCEDM^{n,d,0}) = \matrank(\PointMat^{\top} \PointMat) \leq d$ \citep{dokmanic2015euclidean}.
\subsection{Rank of the Euclidean Distance Matrices with Fault Satellites} 
In \citep{Knowles2023}, they observed that there are more than three nonzero singular values when the GCEDM is constructed from an EDM with fault satellites. An example is provided in Figure \ref{fig:eigval_logplot}.
However, mathematical proofs were not provided in their paper. 
In this section, we provide and prove several properties of the EDM and GCEDM that are corrupted with faults and measurement noises. In the following proofs, we assume $\bar{f}$ is constant for all satellites.

\begin{figure}
    \centering
    \begin{subfigure}[b]{0.46\textwidth}
        \centering
        \includegraphics[width=\textwidth]{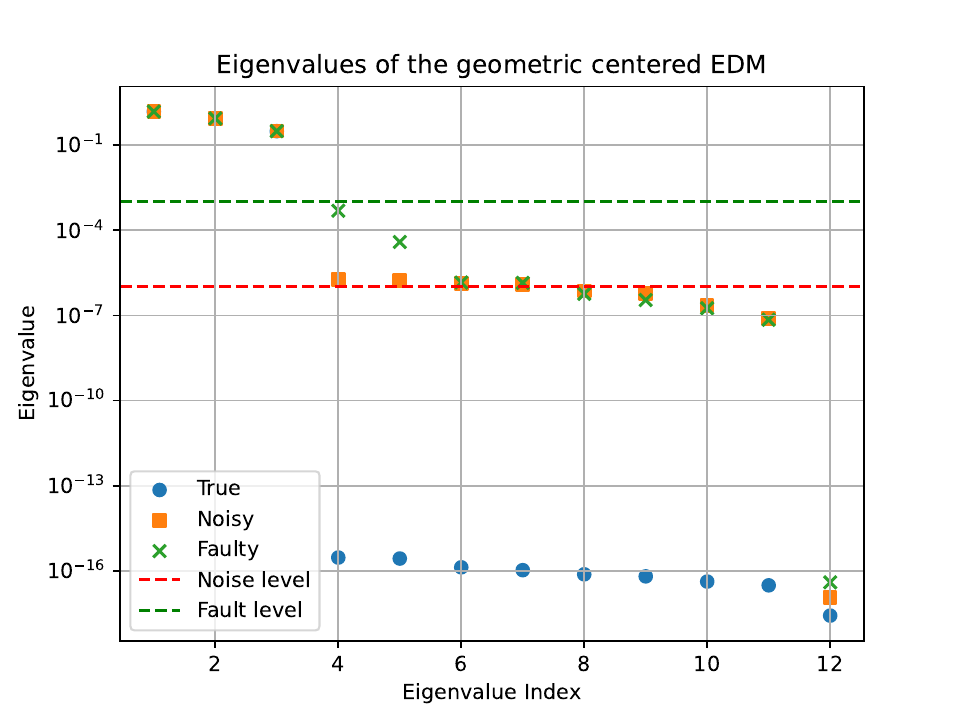}
        \subcaption{Singular values of the geometric centered EDM when 1 fault satellite exists. The fourth and fifth singular values increase compared to the non-fault case when the fault magnitude is sufficiently larger than the noise magnitude.}
    \end{subfigure}
    \hfill
    \begin{subfigure}[b]{0.46\textwidth}
        \centering
        \includegraphics[width=\textwidth]{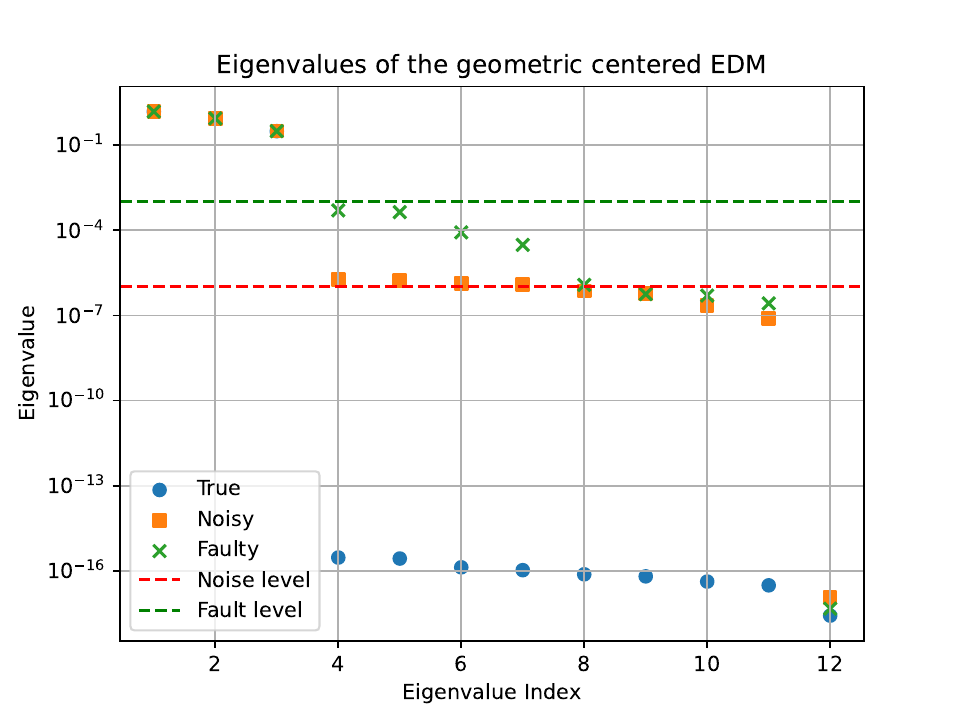}
        \subcaption{Singular values of the geometric centered EDM when 2 fault satellites. The fourth, fifth, sixth, and seventh singular values increase compared to the non-fault case when the fault magnitude is sufficiently larger than the noise magnitude.}
    \end{subfigure}
    \caption{The log plot of the singular values of an example geometric centered EDM for $n=12, d=3$. The singular values are ordered in the descending order of its magnitude. The blue points show the singular values when there is no fault or noise ($\GCEDM^{12, 3, 0}$), the orange points show the singular values when there is noise (of scale $\sigma_w = 10^{-6}$) added to the range measurements ($\GCEDMnoisy^{12, 3, 0}$), and the green points show the singular values when both noise and fault (of scale $\bar{b} = 10^{-3}$) is added ($\GCEDMnoisy^{12, 3, 2}$). The 12 points are identical for all cases, which are sampled randomly inside a unit cube. When noise is present, the singular values increase except the last singular value, which is close to zero (theoretically zero). When $m$ faults with magnitudes that are larger than the noise magnitude are added, the singular values up to the $d + 2m$ index increase compared to the non-fault case.}
    \label{fig:eigval_logplot}
\end{figure}

\begin{proposition}
    \label{prop:edmrank}
    The rank of the EDM $\EDM^{n,d,m}$ satisfies
    \begin{equation}
        \matrank(\EDM^{n, d, m}) \leq \min(d + 2 + 2m, n) 
    \end{equation}
\end{proposition}

\begin{proof}
    See Appendix~\ref{sec:app1}.
\end{proof}

\begin{proposition}
    \label{prop:geocenteredm}
    The rank of the GCEDM $\GCEDM^{n,d,m} = -\frac{1}{2} J^n \EDM^{n, d, m} J^n$  satisfies
    \begin{equation}
        \matrank(\GCEDM^{n,d,m}) \leq \min(d + 2m, n - 1) 
    \end{equation}
\end{proposition}

\begin{proof}
    See Appendix~\ref{sec:app2}.
\end{proof}

In practice, we observe that the relation in Proposition~\ref{prop:geocenteredm} holds with equality except for degenerate cases, such as those noted in Proposition~\ref{prop:edmrank}.
Specifically, the double centering will remove the sparsity pattern 
observed in Proposition~\ref{prop:edmrank}, but it will not change the rank if $2m < n - 1$.
Moreover, although $\GCEDM^{n,d,0}$ is positive semi-definite \citep{dokmanic2015euclidean}, $\GCEDM^{n,d,m}$ is not guaranteed to be positive semi-definite and we observe both positive and negative eigenvalues in practice.
\begin{corollary}
\label{theorem:noisy}
    The GCEDM $\GCEDMnoisy^{n, d, m} = -\frac{1}{2} J_n \EDMnoisy^{n, d, m} J_n$ constructed from an EDM with the edges corrupted by Gaussian noise almost surely has rank
    \begin{equation}
        \matrank(\GCEDMnoisy^{n,d,m}) = n - 1 
    \end{equation}  
\end{corollary}

\begin{proof}
    See Appendix~\ref{sec:app3}.
\end{proof}


\subsection{Distribution of the Singular Values of the Geometric Centered EDM}
\label{sec:distribution_singular}

The propositions proved in the previous sections indicate that we can detect fault satellites by observing the increase in the 4th and 5th singular values of the GCEDMs constructed from the ISR measurements. 
Following the work by ~\citep{knowles2024greedy}, we use the following test statistics $\teststatistic$ to monitor if a fault satellite exists in a given graph.
\begin{equation}
    \teststatistic = \frac{\lambda_4 + \lambda_5}{\lambda_1}
\end{equation}
where $\lambda_i$ is the $i$th singular value of the geometric centered EDM.
If the ISR measurements are completely noiseless, we can detect if a fault satellite exists within the graph by checking if $\teststatistic$ of the GCEDMs is not 0, since GCEDMs have rank 3 when no fault exists. 
However, when the ISR measurements are noisy, the 4th and 5th singular values increase regardless of faults, as shown in Figure \ref{fig:eigval_logplot} and Corollary \ref{theorem:noisy}. 
Therefore, to detect faults under the presence of faults, we need to 
set a threshold based on the singular value distributions of the noisy (but non-fault) GCEDMs that can separate the fault from the noise.

The distribution of the singular values of the three different patterns of noisy GCEDMs is shown in Figure 
\ref{fig:singular_value_distribution}. 
As illustrated in Figure \ref{fig:singular_value_distribution}, the distribution of $\teststatistic$ is affected by the relative geometry of the satellites. When the satellites are spread in 3D space, the distribution of $\teststatistic$ shifts to the right as the noise magnitude increases, as shown in Figure \ref{fig:topology1} and \ref{fig:topology2}. In both non-fault and fault cases, the distribution can be well approximated by a gamma distribution. 
It is worth mentioning that for some geometries, it is not possible to distinguish fault by looking at the singular values because the distribution does not change in the presence of fault, as shown in Figure \ref{fig:topology3}. 
This corresponds to the case where the non-fault satellites are in the same plane. 
This observation indicates that it is better to have satellites on diverse orbit planes to effectively detect faults.

\begin{figure}[htb!]
\centering
\begin{subfigure}[b]{0.31\textwidth}
  \includegraphics[height=50mm]{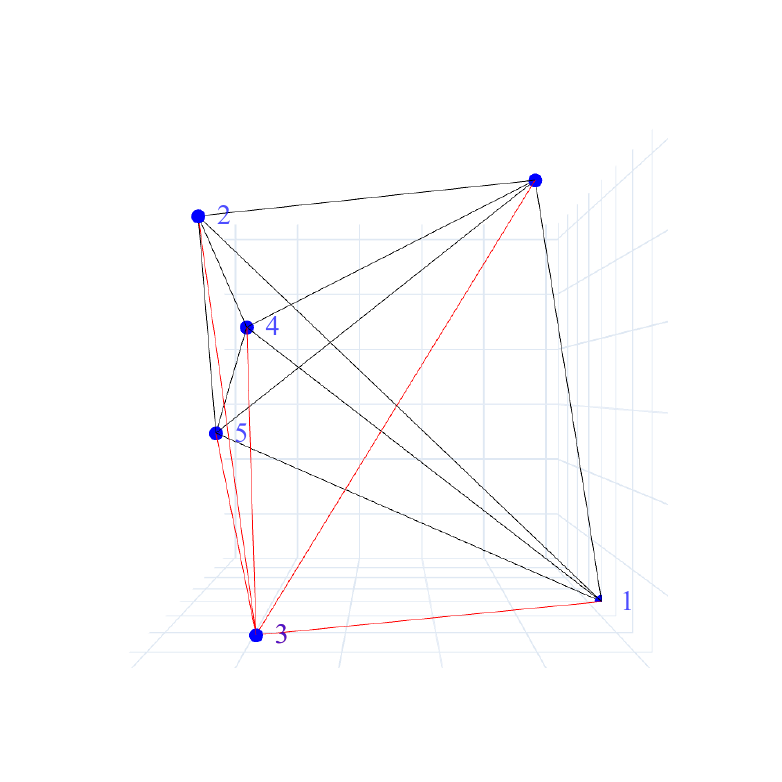}
\end{subfigure}
\hfill
\begin{subfigure}[b]{0.31\textwidth}
  \includegraphics[height=50mm]{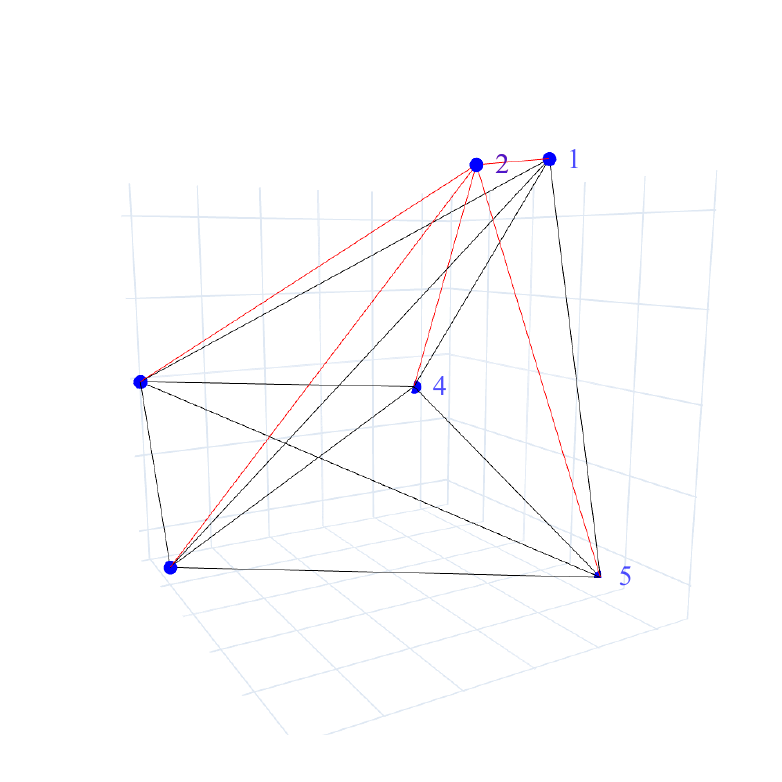}
\end{subfigure}
\hfill
\begin{subfigure}[b]{0.31\textwidth}%
  \includegraphics[height=50mm]{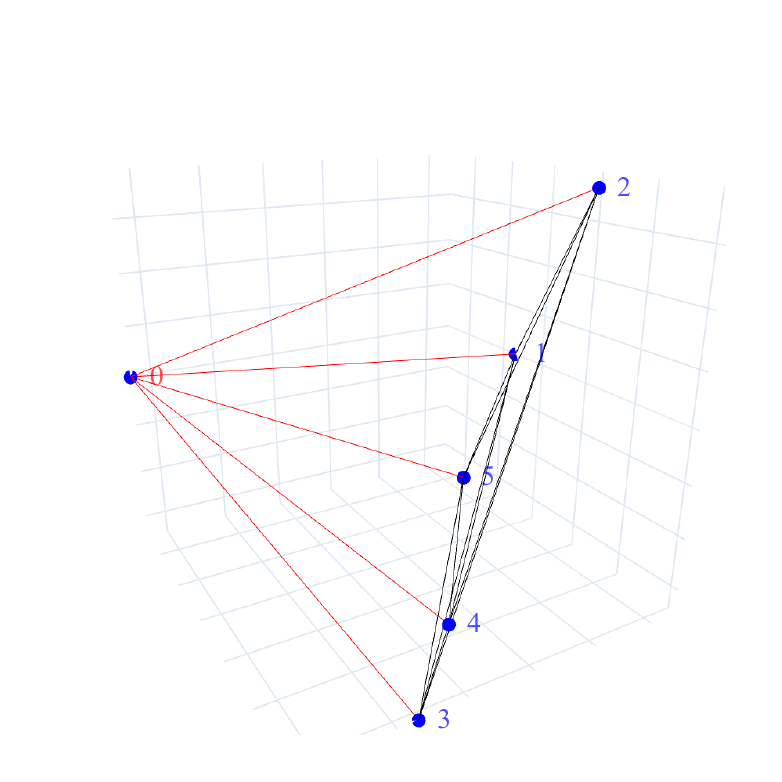}
\end{subfigure}
\begin{subfigure}[b]{0.31\textwidth}
  \includegraphics[width=0.99\linewidth]{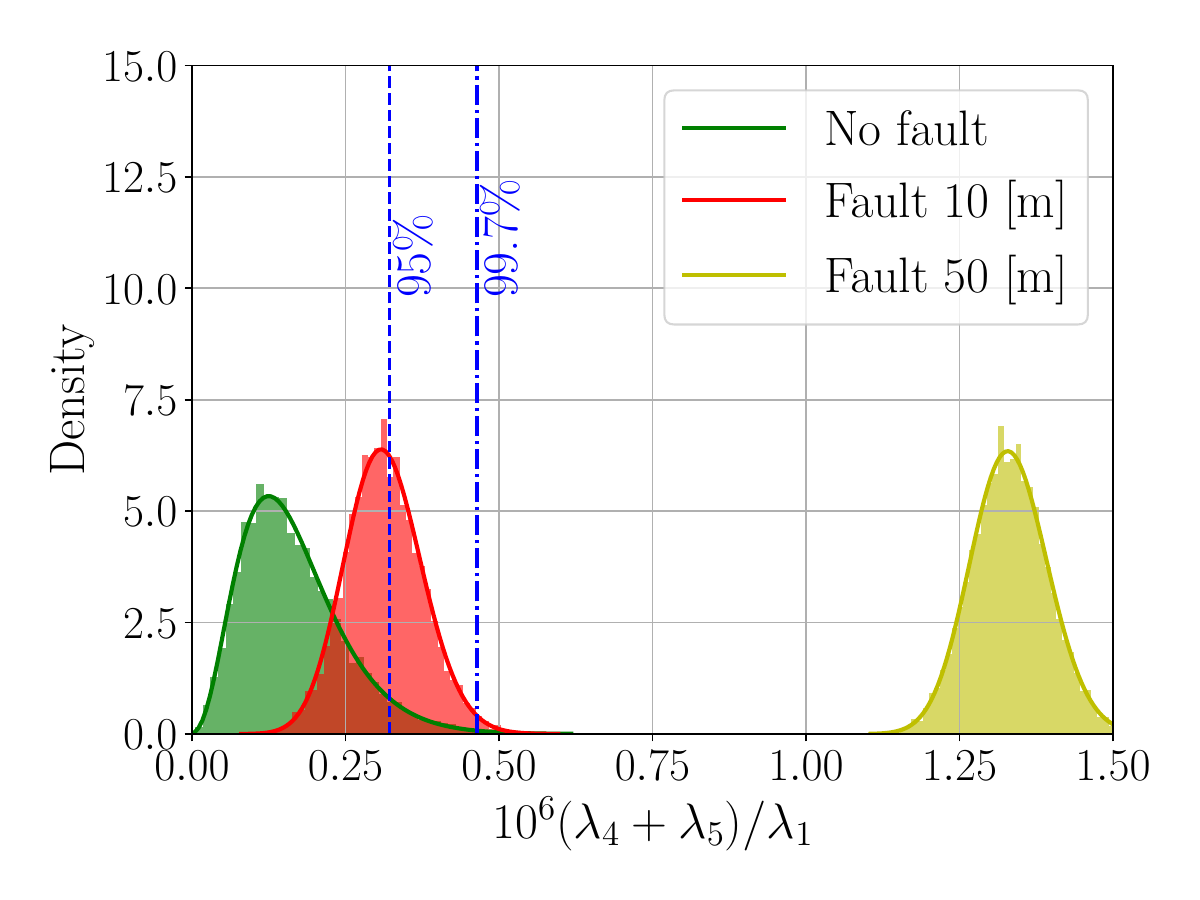}
  \caption{Case 1: A geometry with good fault detectability. There is a higher possibility of being able to detect faults with a 10~\si{\meter} magnitude fault, with its test statistic distribution shifting larger to the right then other two cases.}
  \label{fig:topology1}
\end{subfigure}
\hfill
\begin{subfigure}[b]{0.31\textwidth}
  \includegraphics[width=0.99\linewidth]{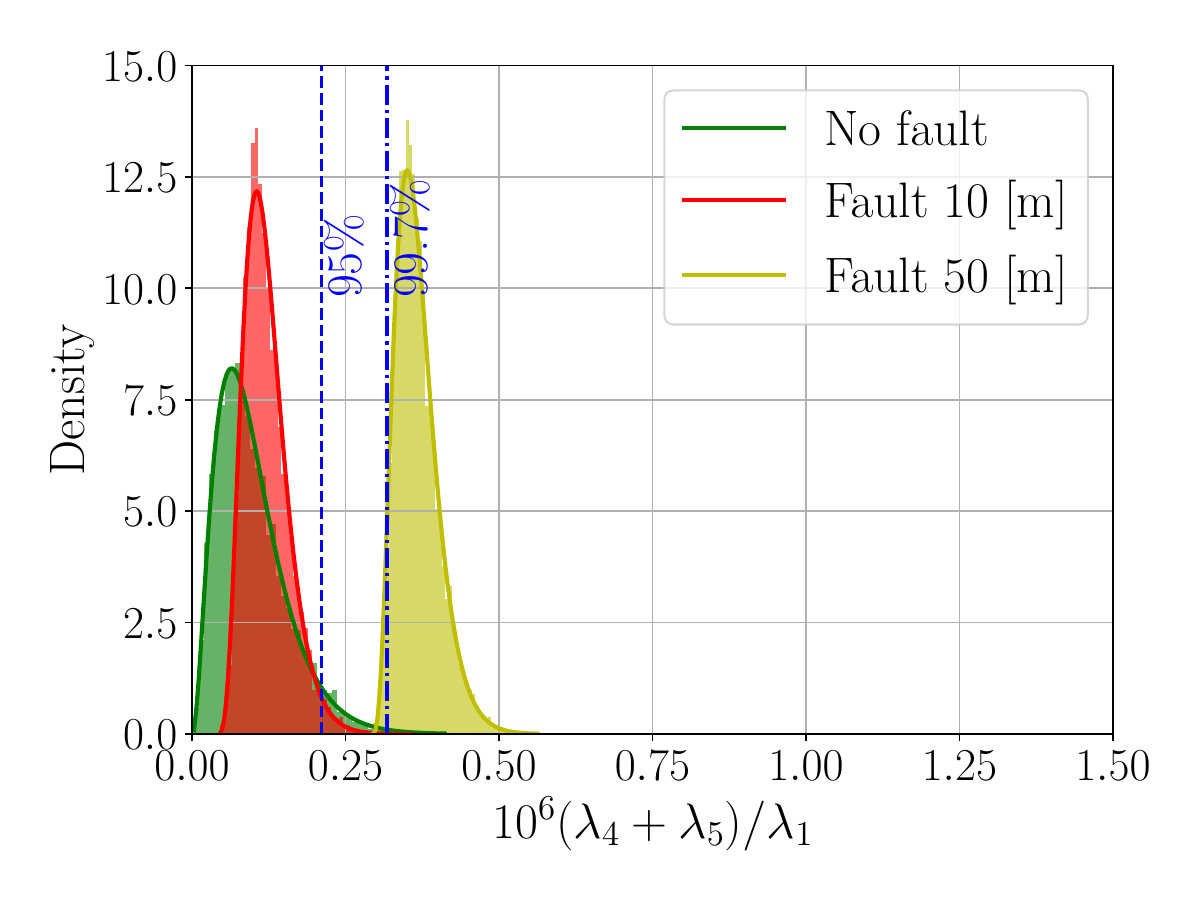}
  \caption{Case 2: A geometry with medium fault detectability. It is difficult to detect faults with 10~\si{\meter} magnitude since the increase in the test statistic is limited compared to the tails of the distribution in the non-fault case. }
  \label{fig:topology2}
\end{subfigure}
\hfill
\begin{subfigure}[b]{0.31\textwidth}%
  \includegraphics[width=0.99\linewidth]{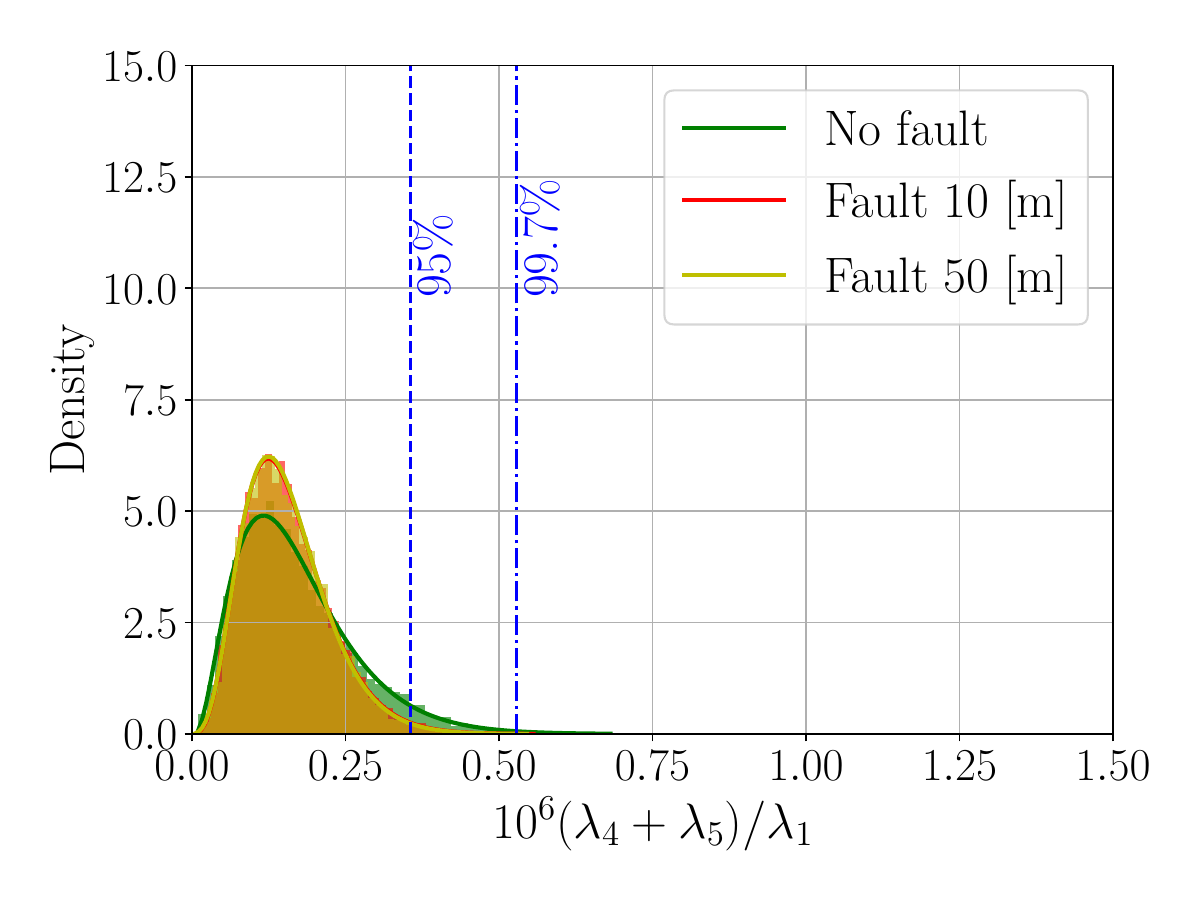}
  \caption{Case 3: A geometry with no fault detectability. The test statistic does increase in fault cases. This is due to the non-preferable 3D topology, where satellites excluding the fault satellites are nearly in the same plane. }
  \label{fig:topology3}
\end{subfigure}

\caption{Three different graph topologies (top row) and the corresponding distribution of the test statistic $\teststatistic = \frac{\lambda_4 + \lambda_5}{\lambda_1}$ (bottom row). The subgraphs are sampled from the lunar constellation shown in Figure \ref{fig:orbit_moon_case}. In the top row, the edges connected to the fault satellites are shown in red. In the bottom row, the histogram of $\teststatistic$ is generated via 10000 Monte-Carlo simulations for different range noises sampled from a Gaussian noise with a standard deviation of 1~\si{\meter}. For each case, we ran the Monte-Carlo simulation for three different fault magnitudes of 0~\si{\meter} (non-fault), 10~\si{\meter}, and 50~\si{\meter}, and fitted the test statistic distribution with a gamma distribution. The presence of a fault satellite changes the distribution of the test statistic $\teststatistic$. }
\label{fig:singular_value_distribution}

\end{figure}

\section{Redundantly Rigid Graph-Based Fault Detection}
\label{sec:algorithm}

\subsection{Clique Listing}
\label{sec:clique-finding}
In order to construct an EDM and GCEDM, we need a set of range measurements between all pairs of $n$~satellites. However, due to the occultation by the planetary body and attitude constraints, the satellites in the constellation are not fully connected with ISRs in most cases.
Therefore, the first step of the fault detection algorithm is to find a set of $k$-clique subgraphs, which are fully connected subgraphs of $k$~satellites. 
For the graph to be 2-vertex redundantly rigid, we need a subgraph with $k \geq 5$.
An example of 5-cliques subgraphs is shown in Figure \ref{fig:cliques_example}.

Various exact algorithms to find all $k$-cliques have been proposed~\citep{Li2020OrderingHF}. In this paper, we used the Chiba-Nishizeki Algorithm (Albo) due to its simplicity of implementation~\citep{Chiba1985}. For each node $v$, Arbo expands the list of $k$-cliques by recursively creating a subgraph induced by $v$'s neighbors. The readers are referred to \citep{Li2020OrderingHF} for the pseudocode and summary of the $k$-clique listing algorithm.
Note that the clique finding algorithm does not have to be executed online. Since the availability of the ISRs can be predicted beforehand using a coarse predicted orbit, we can estimate the topology of the ISRs for future time epochs. We can compute the list of $k$-cliques of the predicted topologies for future time epochs (e.g., one orbit), and uplink them to satellites intermittently. If some of these links were actually not available in orbit for some reason, we can remove the $k$-cliques that contain the missing links from the list.

\begin{figure}[htb!]
    \centering
    \begin{subfigure}[b]{0.24\textwidth}
        \centering
        \includegraphics[width=0.99\linewidth]{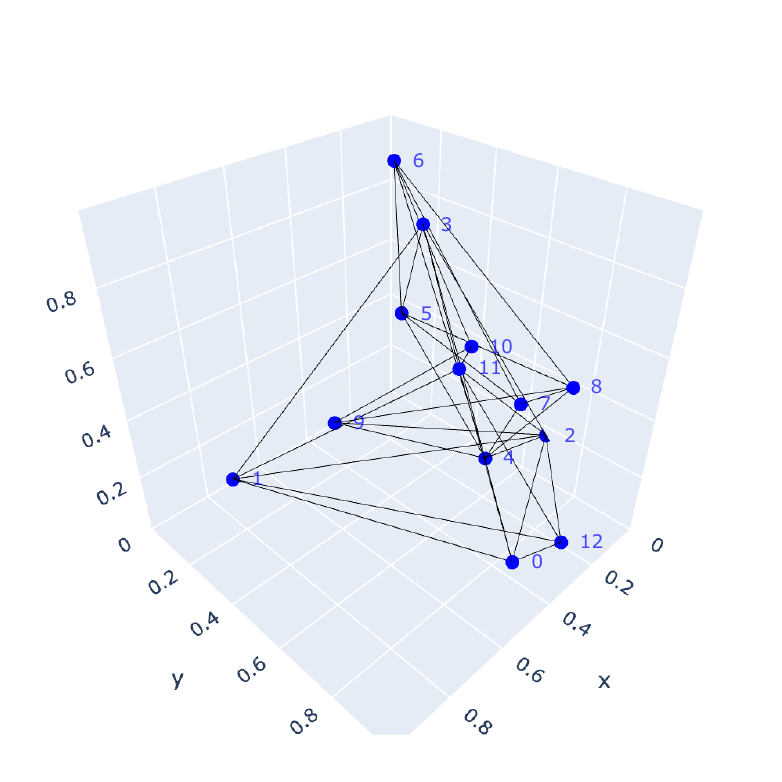}
        \subcaption{The topology of the entire graph}
    \end{subfigure}
    \hfill
    \begin{subfigure}[b]{0.24\textwidth}
        \centering
        \includegraphics[width=0.99\linewidth]{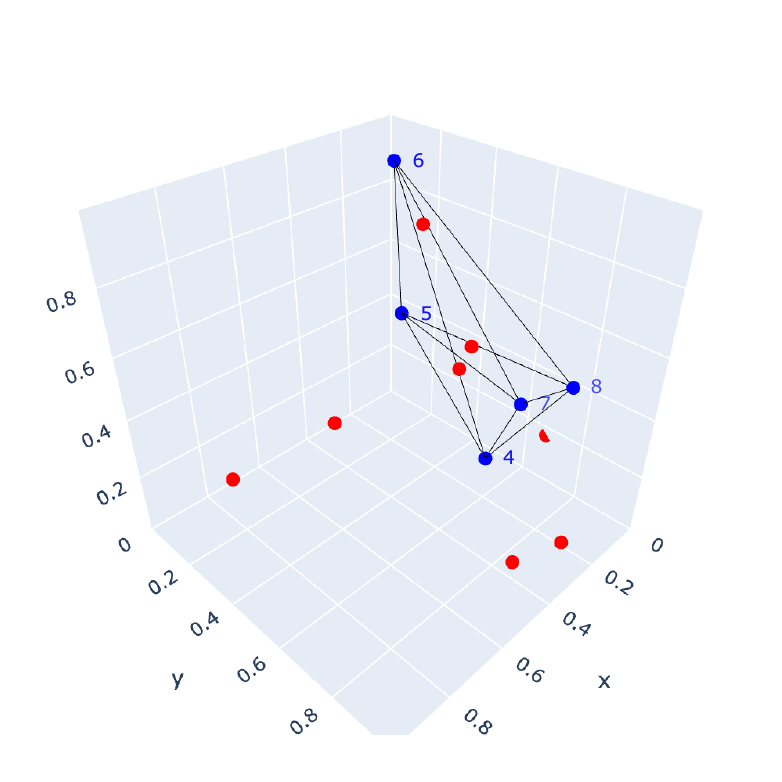}
        \subcaption{5-clique subgraph 1}
    \end{subfigure}
    \hfill
    \begin{subfigure}[b]{0.24\textwidth}
        \centering
        \includegraphics[width=0.99\linewidth]{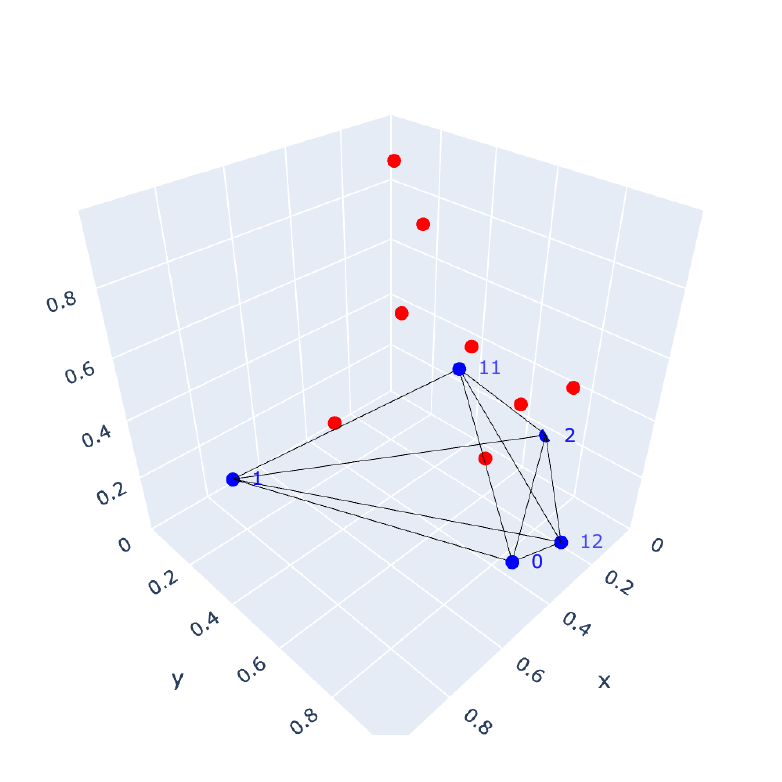}
        \subcaption{5-clique subgraph 2}
    \end{subfigure}
    \hfill
    \begin{subfigure}[b]{0.24\textwidth}
        \centering
        \includegraphics[width=0.99\linewidth]{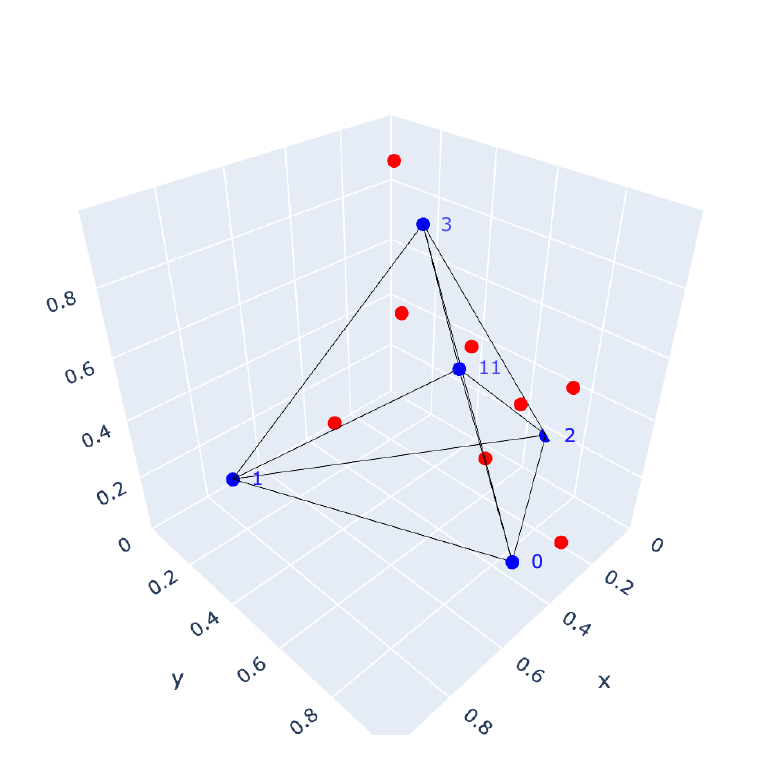}
        \subcaption{5-clique subgraph 3}
    \end{subfigure}
    \caption{Listing 5-cliques within the graph of 13 nodes}
    \label{fig:cliques_example}
\end{figure}

\subsection{Online Fault Detection}
Let $\mathcal{L}^{G}_{k, t}$ be the list of all $k$-cliques of the satellite network at time step $t$, and $N_s$ the total number of satellites.
We propose an online fault detection algorithm, which uses $\mathcal{L}^{G}_{k, t}$. The main routine of the fault detection algorithm is shown in Algorithm \ref{alg:main_loop}. 
The overview of the detection algorithm is as follows. 
First, we judge whether a subgraph ($k$-clique) contains a fault satellite by comparing the user-set threshold $\teststatisticthreshold$ and the observed test statistics value.
If the subgraph is identified as a fault, then we look at the elements of $\bm{u}_4$ (the 4th column of the orthogonal matrix obtained by SVD) to find which satellite within the subgraph is generating the fault.
In particular, we look for the row with the largest magnitude. The satellite corresponding to this row is registered as a ``fault candidate," and a counter will be incremented. 
The satellite that has been registered as a fault candidate for most of the time is considered as a fault satellite and is removed from the entire graph.
This procedure is repeated until there are no more faults.
The key difference compared to the greedy GPS fault-detection algorithm proposed by~\citep{knowles2024greedy} is that we are using multiple subgraphs (EDMs) to detect faults, instead of using a (large) single fully-connected graphs.

The proposed fault detection algorithm has several hyper-parameters that affect its performance. Below, we provide a guide on how the hyper-parameters should be determined.

\textbf{Clique Size $k$}: The size of the clique needs to be 5 or larger to compute the test statistics. The clique size $k=5$ maximizes the number of cliques that can be used for fault detection. However, from observation, we found that when we use $k=5$, it is difficult to determine the fault satellite by looking at the index $\text{argmax}_i (\bm{u}_{4})_{i}$. 
Therefore, we propose to use $k=6$ cliques for fault detection.
An example is shown in Figure \ref{fig:singular_value_cliques}. 

\textbf{Detection time interval $DI = |\bm{t}_{fd}|$}: This controls the number of timesteps used to identify fault satellites. 
In general, the true positive rate (TPR) will increase with larger DI, because the total number of subgraphs increases and the probability that large test statistics will be observed will increase.

\textbf{Least required number of fault subgraphs $\delta_{nf}$}: The algorithm will not detect faults if the total number of subgraphs detected as faults does not exceed this threshold. 
We set this threshold to avoid judging faults from a small number of samples and to reduce the false alarm rate. 

\textbf{Mimimum fault detection ratio $\delta_{rf}$}: The algorithm will not detect any satellite $i$ as having a fault if the ratio of the number of times satellite $i$ detected as faults with respect to the total number of fault subgraphs does not exceed this threshold. This parameter needs to be set based on the maximum number of fault satellites ($n_{fs, max}$). In particular, we require $\delta_{rf} < \frac{1}{n_{fs, max}}$ 

\textbf{Test statistics threshold $\teststatisticthreshold$}: This controls the trade-off between increasing the TPR and decreasing the FPR. 
One way to set this threshold is to sample a sufficiently large number of singular values from subgraphs over different epochs, and compute its $x$ percentile value close to the tail (e.g., 99 percentile).
When a small $\teststatisticthreshold$ is used, more subgraphs are listed as fault subgraphs, which increases the probability of fault satellites being detected, but also increases the probability of listing normal satellites as a fault. 
The threshold $\teststatisticthreshold$ and the detection time interval $DI$ need to be determined by pair by the system designer to get the desired balance between TPR and FPR.

\SetKwComment{Comment}{/* }{ */}
\RestyleAlgo{ruled}
\DontPrintSemicolon
\SetKwFunction{FMain}{Main}
\SetKwFunction{FSVD}{}
\SetKwFunction{FFaultDetection}{FaultDetection}
\SetKwProg{Fn}{Function}{:}{}
\SetKwProg{Pn}{Function}{:}{\KwRet}

\begin{algorithm}[ht!]
\caption{Main Loop of the Fault Detection Algorithm }
\label{alg:main_loop}
    \textbf{Given:} \;
    $\quad\quad k$ \hspace{23mm} $\triangleright$ \ \text{Hyperparameter: size of the clique}  \;
    $\quad\quad \mathcal{L}^{G}_{k, t}$ \hspace{19mm} $\triangleright$ \ \text{List of k-clique subgraphs for timestep $t$}  \;
    $\quad\quad \bm{t}_{fd}$ \hspace{20.2mm} $\triangleright$ \ \text{List of timesteps the fault detection is performed at. Its size is a hyper-parameter.}  \;
    $\quad\quad \delta_{nf}$ \hspace{20mm} $\triangleright$ \ \text{Hyper-parameter: Least number of fault subgraphs required to identify fault }  \;
    $\quad\quad \delta_{rf}$ \hspace{20.4mm} $\triangleright$ \ \text{Hyper-parameter: Minimum fault detection ratio for the satellite to be judged as fault}  \;
    $\quad\quad \teststatisticthreshold$ \hspace{20.0mm} $\triangleright$ \ \text{Hyper-parameter: Test statistics threshold}  \;
                
    \vspace{2mm}
    $\mathcal{L}^F \gets \emptyset $  \hspace{20.5mm}  $\triangleright$ \ \text{List of fault satellites} \;
    $T \gets 0$  \hspace{22.7mm}  $\triangleright$ \ \text{Termination flag, terminate if $1$}   \;
    \While{$T == 0$}{
        $\mathcal{L}^{G}_{k, t}, \mathcal{L}^F, T \gets$ \FFaultDetection{$\mathcal{L}^{G}_{k, t}, \mathcal{L}^F, T, \bm{t}_{fd}, \delta_{nf}, \delta_{rf}, \teststatisticthreshold$}  \hspace{20mm}  $\triangleright$ See Algorithm \ref{alg:fault_detection_function} \;
    }
    \Return{$\mathcal{L}^F$}
\end{algorithm}

\begin{algorithm}[ht!]
\caption{Fault detection function}
\label{alg:fault_detection_function}

\Pn{\FFaultDetection{$\mathcal{L}^{G}_{k, t}, \mathcal{L}^F, T, \bm{t}_{fd}, \delta_{nf}, \delta_{rf}, \teststatisticthreshold$}}{
    \vspace{1mm}
    $\bm{N^f} \gets \zeromat \in \N^{N_s}$  \hspace{10mm}  $\triangleright$ \text{Vector where each element $i$ shows  number of times that satellite $i$ was detected as fault}  \;

    \vspace{2mm}
    $\triangleright$ Step 1: Counting the number of faults \;

    \For {$t \in \bm{t}_{fd} $}{
        \For {$G \in \mathcal{L}^{G}_{k, t}$}{
            \vspace{2mm}
            $\bm{s} \gets \{s_i \ | \  s_i \in G \} $  \hspace{3mm} $\triangleright$ List of satellites in the subgraph G\;
            \vspace{2mm}
            $\triangleright$ Step 1-1: Construct GCEDM from the observed ranges $r$ \;
            {
                $\EDMnoisy_{ij} \gets \satrange_{s_i s_j} \ (\EDMnoisy_{ij} \gets 0$ for $i=j$)  \;
                $\GCEDMnoisy \gets -\frac{1}{2} \GCMat^k \EDMnoisy \GCMat^k$ \;
            }
            \vspace{2mm}
            $\triangleright$ Step 1-2: Singular Value Decomposition (SVD) of the GCEDM \;
            $\bm{U}, \bm{\Sigma}, \bm{V} \leftarrow \text{SVD}(\GCEDMnoisy)$\;
            $\lambda_1, \ldots, \lambda_{k} \gets \text{diag}(\bm{\Sigma})$ \;
            $\bm{U} = \begin{bmatrix} \bm{u_1} & \cdots & \bm{u_k} \end{bmatrix}  \quad (\bm{u}_i \in \R^{k}$) \;
            \vspace{2mm}
            $\triangleright$ Step 1-3: Predict if the subgraph contains a fault satellite \;
            $\teststatistic \gets \frac{\lambda_4 + \lambda_5}{\lambda_1}$ \hspace{46.2mm}  $\triangleright$ Observed test statistic  \;
            \vspace{2mm}
            \If{$\teststatistic > \teststatisticthreshold$}{
                \vspace{1mm}
                $\triangleright$ Step 1-4: Predict which is the fault satellite by looking at the 4th column vector of the orthogonal matrix \;
                $j \gets \text{argmax}_i |(\bm{u}_4)_i|$\;
                $(\bm{N^f})_{s_j} \gets (\bm{N^f})_{s_j} + 1$\;
            }
        }
    }

    $(\bm{R^f})_i = \frac{(\bm{N^f})_i}{\sum_i{(\bm{N^f})_i}} \in \R^{N_s} \quad $  $\triangleright$ Ratio of fault detection of each satellite over all fault detections \;\;
    
    $\triangleright$ Step 2-1: Terminate if we do not have enough fault satellites \;
    \If{$\sum{\bm{N^f}} \leq \delta_{rf}$}{
        $T \gets 1$\;
        \Return{$\mathcal{L}^{G}_{k,t}, \mathcal{L}^F, T$}
    }
    \vspace{1mm}
    $\triangleright$ Step 2-2: Terminate if no satellites have a fault detection ratio over a threshold (i.e., no satellites were sufficiently detected as fault compared to the others) \;
    \If{$\max{(\bm{R^f})} \leq \delta_{nf}$}{
        $T \gets 1$\;
        \Return{$\mathcal{L}^{G}_{k,t}, \mathcal{L}^F, T$}
    }
    \vspace{2mm}
    $\triangleright$ Step 3: Remove the satellite which was detected as fault for most times \;
    $s_{\text{fault}} \gets \text{argmax}_{j} \ (\bm{N^f})_{j}$ \;
    $\mathcal{L}^{G}_{k,t} \gets \{G \ | \ G \in \mathcal{L}^{G}_{k,t}, s_{\text{fault}} \notin G \} \quad \triangleright$  \text{Remove all cliques that contains the fault satellite} \;
    $\mathcal{L}^{F} \gets \mathcal{L}^{F} \cup \{ s_{fault} \} \quad \triangleright$ \text{Add the fault satellite to the list} \;
    $T \gets 0$ \;
    \Return{$\mathcal{L}^{G}_{k,t}, \mathcal{L}^F, T$} 
}
\end{algorithm}

\begin{figure}[htb!]
\begin{subfigure}[b]{0.38\linewidth}
    \centering
    \includegraphics[width=\textwidth]{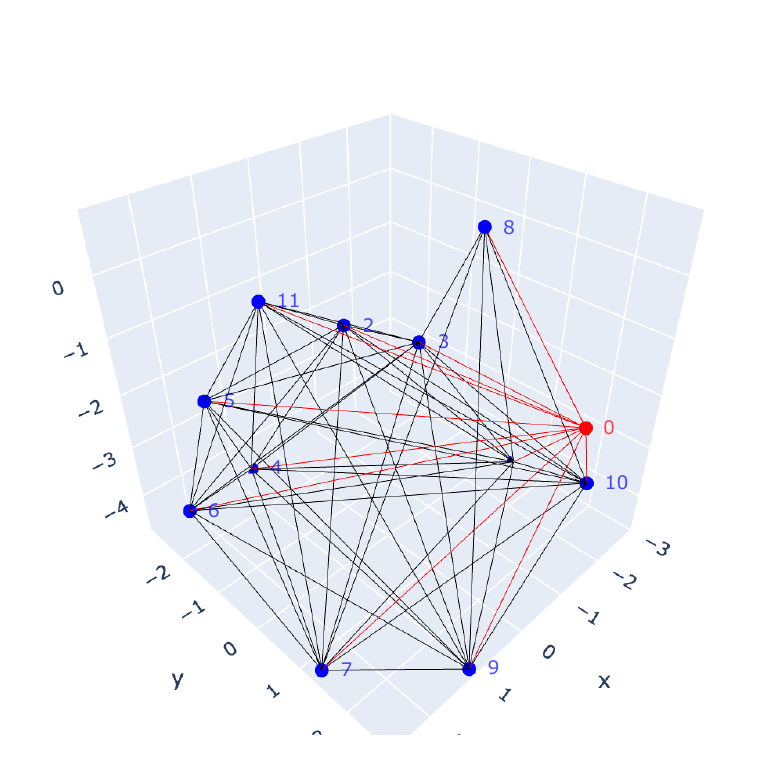}
    \subcaption{The geometry and topology of the satellite links. The fault satellite 0 and its connected edges are shown in red. The units are normalized by the Moon radius.}
\end{subfigure}
\hfill
\begin{subfigure}[b]{0.60\linewidth}
    \centering
    \includegraphics[width=\textwidth]{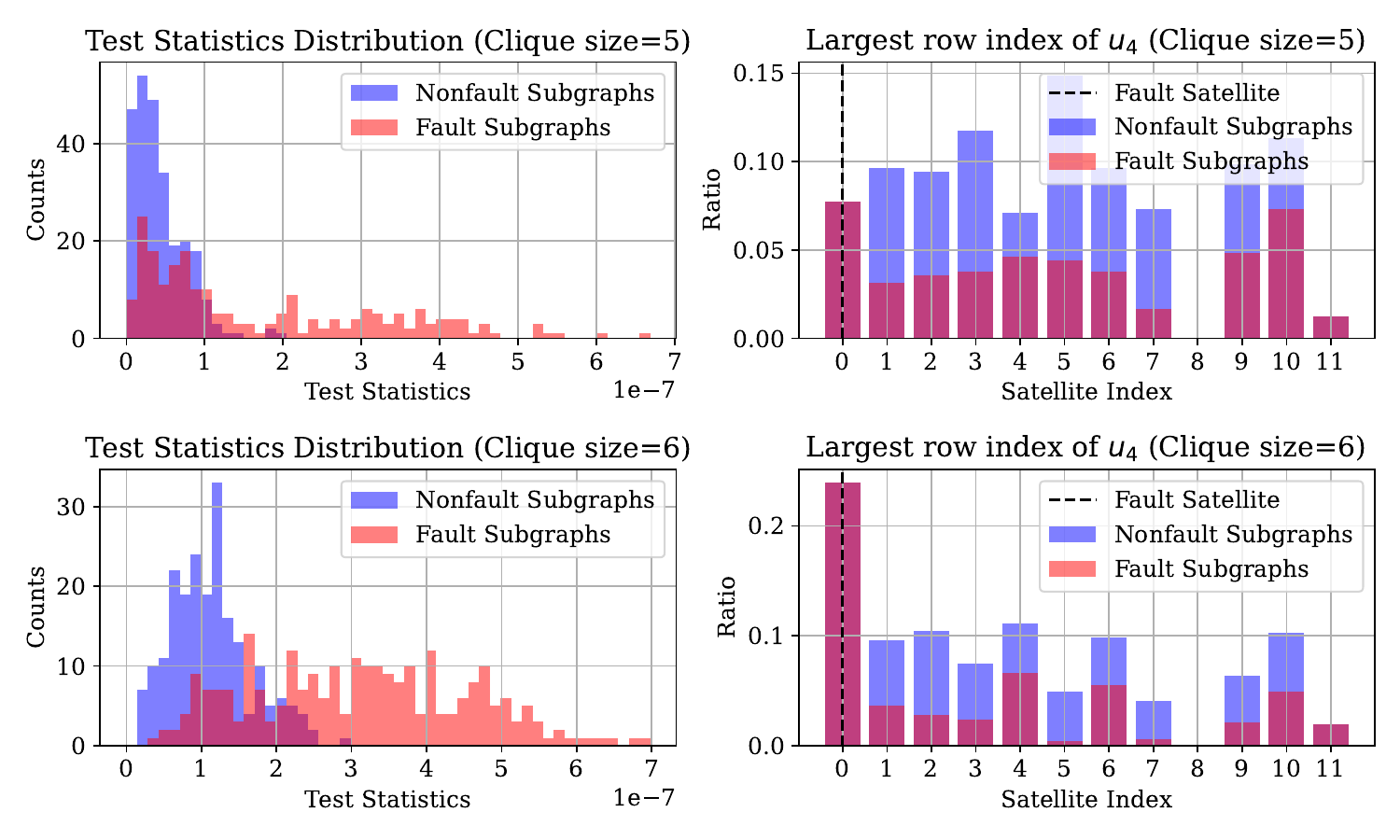}
    \subcaption{The distribution of the test statistic $\teststatistic$ and the largest row of the 4th column of the orthogonal matrix ($\text{argmax}_i |(\bm{u}_4)_{i}|$). When the clique size is 6, we can detect the fault satellite $s_f = 0$ by counting the number of subgraphs where the largest index of $\bm{u}_4$ corresponds to $s_f$.}
    \label{fig:singular_vectors}
\end{subfigure}
\caption{The distribution of the test statistic $\teststatistic$ and the row index which has the maximum element of the 4th singular vector $\bm{u}_4$. The geometry and topology of the links were taken from a timestep for the lunar constellation in Figure \ref{fig:orbit_moon_case}. The noise of $\sigma_w=1$ m and fault of $\bar{b}=20$ m is added. As shown in Figure \ref{fig:singular_vectors}, we can detect faults from singular vectors more robustly, by using 6-cliques compared 5-cliques.  }
\label{fig:singular_value_cliques}
\end{figure}

\section{Satellite Fault Detection Simulation}
\label{sec:simulation}

\subsection{Simulation Configuration and Evaluation Metrics}
We validate our proposed algorithms for the constellations around the Moon. 

\subsubsection{Constellation}
We consider an Elliptical Lunar Frozen Orbit (ELFO) constellation with a total of 12 satellites equally spread to 4 different orbital planes. 
The orbital elements of the constellation are shown in Table \ref{tab:moon_constellation}, and
the 3-dimentional plots of the constellation and the available ISls at $t=0$ are shown in Figure \ref{fig:moon_orbit}. 
The orbits are propagated in the two-body propagator without perturbations.

The visibility of the links is calculated assuming that the links are visible when it is not occulted by the Moon.
The number of 6-cliques in the constellation with the number of subgraphs containing satellite 1,2,3 (satellites in the first plane) is shown in Figure \ref{fig:moon_nclique}.
Each satellite experiences a time window where they have zero or few self-containing subgraphs, where faults cannot be detected.

\begin{table}[ht!]
    \centering
    \caption{Orbital elements of the ELFO constellation (Case 2)}
    \begin{tabular}{|c|c c c c c c|}  \hline
       Plane & Semi-Major Axis & Eccentricity & Inclination  & RAAN &  Argument of Periapsis & 
       Mean Anomaly  \\
       &  $a$ [km] & $e$ []  & $i$ [deg] & $\Omega$ [deg] & $\omega$ [deg]  & $M$ [deg]       \\ \hline
       1 &  6142.4 & 0.6 & 57.7 &  -90 & 90 & [0, 120, 240] \\
       2 &  6142.4 & 0.6 & 57.7 &  0 & 90 & [30, 150, 270] \\
       3 &  6142.4 & 0.6 & 57.7 &  90 & 90 & [60, 180, 300] \\ 
       4 &  6142.4 & 0.6 & 57.7 &  180 & 90 & [90, 210, 330] \\  \hline
    \end{tabular}
    \label{tab:moon_constellation}
\end{table}

\begin{figure}[tb!]
    \centering
    \begin{subfigure}[b]{0.49\textwidth}
        \centering
        \includegraphics[width=\linewidth]{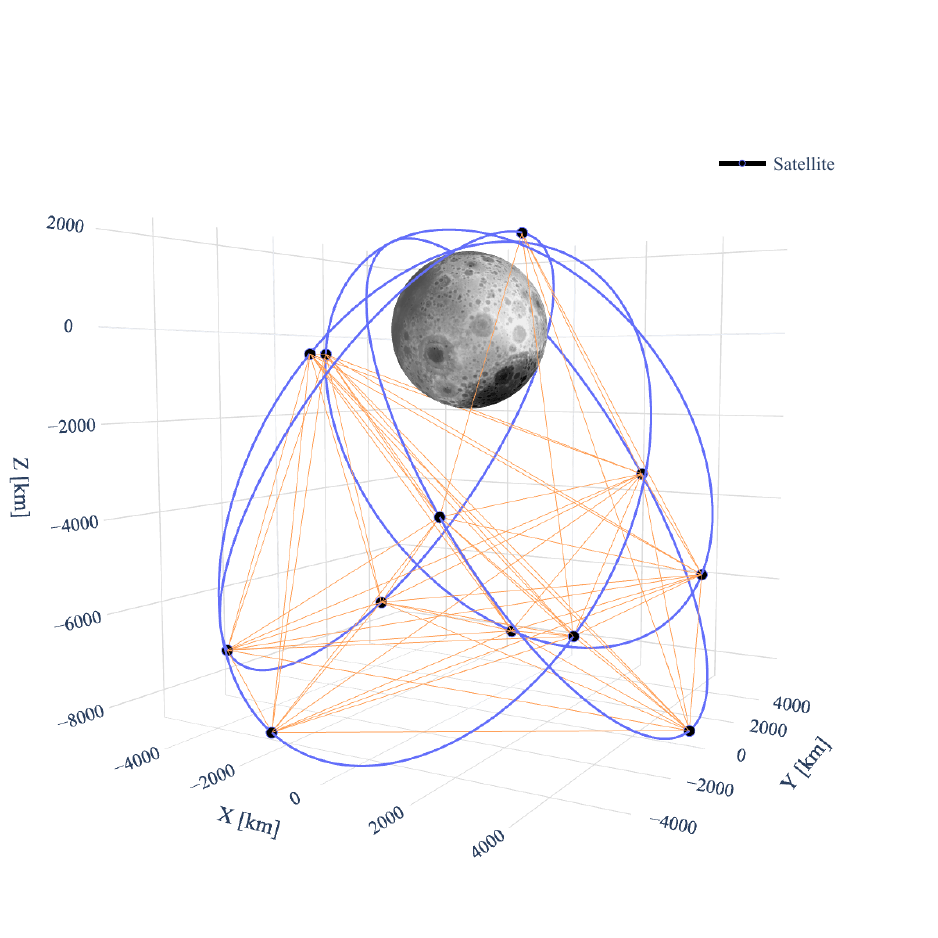}
        \subcaption{Orbital planes with satellite position (black points) and links (orange line) at the initial time epoch. Satellites near the perilune have only 4 or 5 links available due to occultation by the Moon. }
        \label{fig:moon_orbit}
    \end{subfigure}
    \hfill
    \begin{subfigure}[b]{0.49\textwidth}
        \centering
        \includegraphics[width=\linewidth]{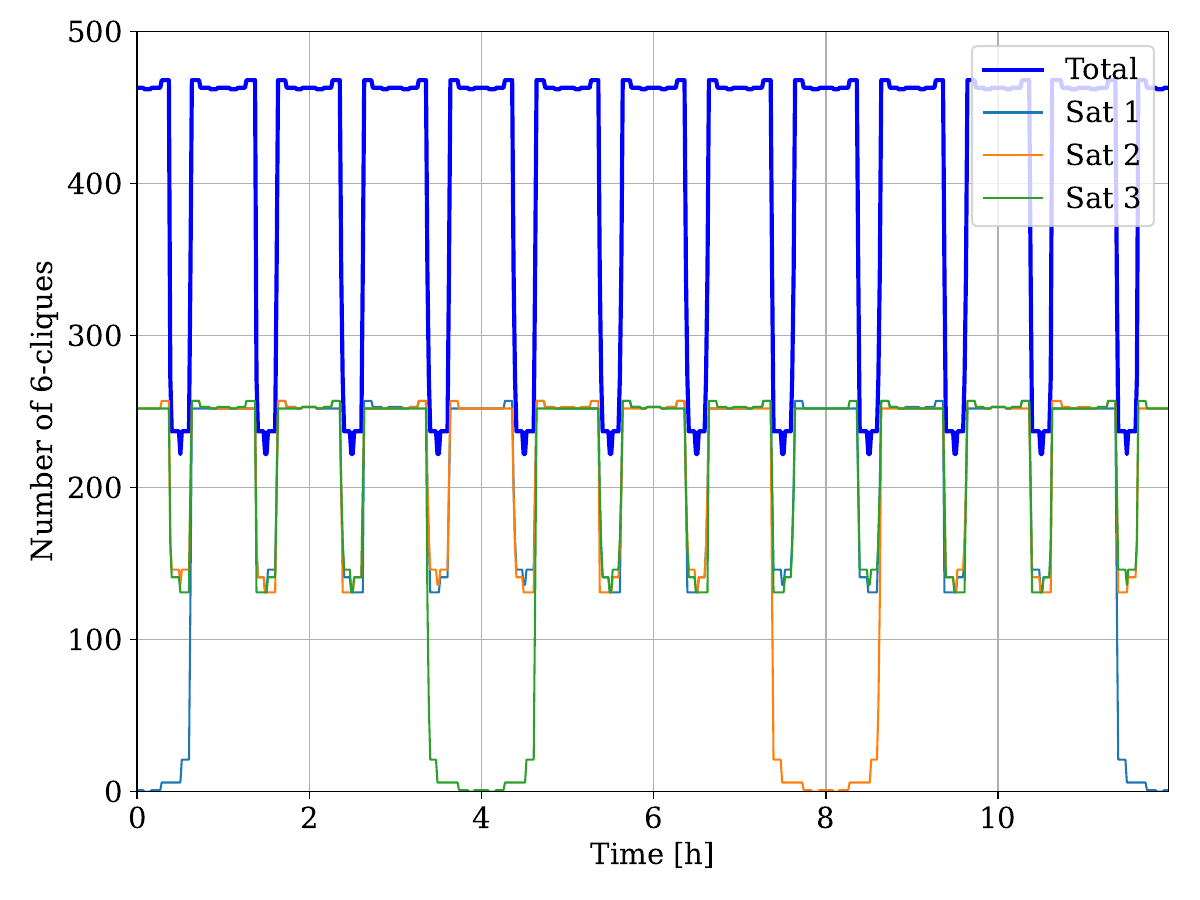}
        \subcaption{Number of 6-cliques for different time epochs. The satellite near the perilune has few or zero 6-cliques containing itself due to the occultation by the Moon. Sat 1, 2, 3 are the 3 satellites in plane 1.}
        \label{fig:moon_nclique}
    \end{subfigure}
    \caption{Elliptical lunar frozen orbit constellation around Moon with 12 satellites. }
    \label{fig:orbit_moon_case}
\end{figure}

\subsubsection{Simulation Parameters}
We assumed that the two-way noise measurement error is $\sigma_w = 1$ m.
We tested 180 combinations of different detection and simulation parameters: 3 different numbers of fault satellites (1, 2, or 3), 5 different fault magnitudes ($\bar{b}$ = 5, 8, 10, 15, 20 m), and 4 different detection time intervals $(DI = 1, 2, 3, 5)$, and 3 different test statistic thresholds $\teststatisticthreshold$. 
The timestep between each detection for $DI \geq 2$ is set to 60 seconds.

The three thresholds are computed by sampling the test statistics from all the (non-fault but noisy) 6-cliques for one orbital period, and computing the $95, 99$, and $99.9$ percentile value. 
The test statistics are sampled every 60s for one orbital period, which results in total of 256742 subgraphs.
The computed thresholds for $95, 99$, and $99.9$ percentile are $\teststatisticthreshold^{95} = 3.58 \times 10^{-7}, \teststatisticthreshold^{99} = 4.57 \times 10^{-7}, $ and $\teststatisticthreshold^{99.9} = 5.86 \times 10^{-7}$, respectively. The distribution and the percentile values of the sampled test statistics are shown in Figure \ref{fig:sampled_test_statistics}.

For each of these cases, we run 500 Monte-Carlo analyses with randomly sampled initial time (within one orbital period) and a set of fault satellites. 
For the other hyper-parameters, we used $k=6, \delta_{nf} = 10, \delta_{rf} = 0.2$.
Note that the seed of the Monte-Carlo simulation is fixed so that the performance for different parameters is compared among the same conditions except the parameters.

\begin{figure}[ht!]
    \centering
    \includegraphics[width=0.6\linewidth]{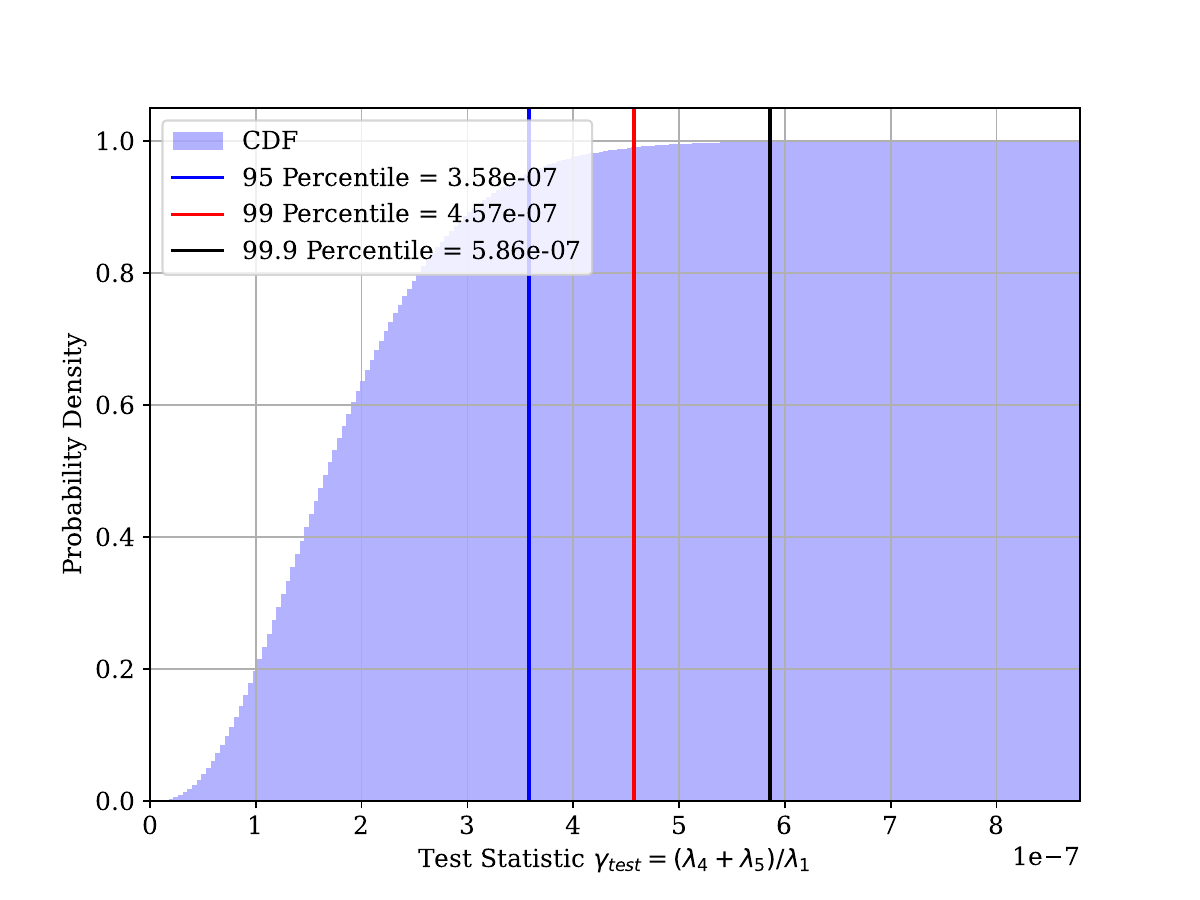}
    \caption{The cumulative distribution of the sampled test statistics $\gamma_{test}$. The test statistics are sampled from all 6-cliques once every 60s for one orbital period. The 95, 99, and 99.9 percentiles are used as thresholds ($\teststatisticthreshold^{95}, \teststatisticthreshold^{99}, \teststatisticthreshold^{99.9}$) to detect faults.}
    \label{fig:sampled_test_statistics}
\end{figure}

\subsubsection{Evaluation Metrics}
We compute the following metrics to evaluate the performance of the fault detection algorithm
\begin{align}
    \text{True Positive Rate or Recall} \ (TPR) &= \frac{TP}{TP + FN} \\
    \text{False Positive Rate} \ (FPR) &= \frac{FP}{FP + TN} \\
    P_4 \ \text{metric} \ (P_4) &= \frac{4 \cdot TP \cdot TN}{4 \cdot TP \cdot TN + (TP + TN) \cdot (FP + FN)} 
\end{align}
where $TP, FN, FP, TN$ is the total number of true positives (fault satellite classified as fault), false negatives (fault satellite classified as non-fault), false positives (non-fault satellite classified as fault), and true negatives (non-fault satellite classified as non-fault). $P_4$ value is the harmonic mean of four key conditional probabilities (precision, recall, specificity, and negative predictive value), and $P_4=1$ requires all of these values to be 1~\citep{Sitarz_2023}.
\subsection{Simulation Result}
\label{sec:moon_sim}
The fault detection results for different combinations of simulation and detector parameters are shown in Table \ref{tab:MoonResult}.
The TPR, FPR, and P4 values improve as fault magnitude becomes larger because it becomes easier for the detector to distinguish faults from noise.
However, the TPR is capped around 0.92, because at any time there will be 1 satellite (8.3$\%$ of 12) located near the perilune, whose fault cannot be detected because of the limited number of self-containing subgraphs.

Figure \ref{fig:comparison_thrshold} and Figure \ref{fig:comparison_dl} show the relationship between the test statistic threshold, detection length, and the three performance metrics.
In general, TPR improves as the test statistic threshold $\teststatisticthreshold$ decreases and detection length $DL$ increases. 
On the contrary, FPR improves as the threshold increases and detection length decreases. 
This is because lowering the threshold and increasing the detection length both increase the number of subgraphs detected as faults, which increases the probability of detecting a faulty satellite but also increases the probability of detecting normal satellites as faults.
The P4 value, which shows the balance, is maximized at higher thresholds and longer DL when fault magnitudes are smaller, and vice versa for larger fault magnitudes. 

Figure \ref{fig:comparison_nfaults} shows the relationship between the number of faults and the three performance metrics, for a fixed detection threshold and detection length. In general, TPR, FPR, and $P_4$ value all improve when the number of faults is smaller, except for TPR and $P_4$ in small fault magnitude cases where it becomes easier to detect at least one fault out of multiple faults.
Our proposed algorithms removes fault satellites in a greedy manner, so when there are multiple fault satellites, the larger perturbation in the 4th and 5th singular value and the 4th singular vector makes it difficult to determine the correct set of fault satellites.
We could potentially improve our detection algorithm by also looking at the singular values of larger indices (e.g., 6th and 7th), but that would require using larger subgraphs for detection.

\begin{table}[ht!]
\centering
\caption{Fault detection results for different numbers of faults, fault magnitude, fault detection threshold, and detection length (DL). The best TPR, FPR, and P4 values for each (number of faults, fault magnitude) pair are highlighted in blue. 
All three metrics improve when the fault magnitude is larger, and the number of faults is smaller.
TPR improves as the threshold decreases and detection length increases, and FPR improves as the threshold increases and detection length decreases. 
The $P_4$ values are better at higher thresholds and longer DL at smaller fault magnitudes, and vice versa for larger fault magnitudes.  }
\begin{tabular}{Ic|c|cI  c|c|cI  c|c|cI  c|c|cI}
\bhline{1pt}
\multirow{3}{*}{Faults} & Threshold & \multirow{2}{*}{DL} & \multicolumn{3}{cI}{True Positive Rate} & \multicolumn{3}{cI}{False Positive Rate} & \multicolumn{3}{cI}{P4 value} \\
 & $\teststatisticthreshold$  &  & \multicolumn{3}{cI}{Fault Magnitude [m]} & \multicolumn{3}{cI}{Fault Magnitude [m]} & \multicolumn{3}{cI}{Fault Magnitude [m]} \\
 & [$\%$] & [] & 5 & 10 & 20 & 5 & 10 & 20 & 5 & 10 & 20 \\
\bhline{1pt}
\multirow{12}{*}{1} & \multirow{4}{*}{99.9} & 1  & 0.006 & 0.374 & 0.864 & \color{blue}{ \textbf{ 0.001 } } & \color{blue}{ \textbf{ 0.000 } } & \color{blue}{ \textbf{ 0.001 } } & 0.023 & 0.698 & 0.955 \\
 &  & 2 & 0.042 & 0.582 & 0.896 & 0.001 & 0.002 & 0.002 & 0.147 & 0.833 & 0.964 \\
 &  & 3 & 0.048 & 0.742 & 0.908 & 0.002 & 0.002 & \color{blue}{ \textbf{ 0.001 } } & 0.165 & 0.908 & \color{blue}{ \textbf{ 0.969 } } \\
 &  & 5 & 0.124 & 0.854 & 0.918 & 0.005 & 0.001 & 0.003 & 0.345 & \color{blue}{ \textbf{ 0.951 } } & 0.968 \\
\cline{2-12}
 & \multirow{4}{*}{99} & 1  & 0.146 & 0.776 & 0.900 & 0.015 & 0.005 & 0.006 & 0.362 & 0.911 & 0.952 \\
 &  & 2 & 0.450 & 0.866 & 0.914 & 0.023 & 0.014 & 0.013 & 0.682 & 0.917 & 0.936 \\
 &  & 3 & 0.530 & 0.874 & 0.924 & 0.032 & 0.026 & 0.019 & 0.712 & 0.888 & 0.922 \\
 &  & 5 & 0.744 & 0.892 & \color{blue}{ \textbf{ 0.926 } } & 0.050 & 0.041 & 0.039 & \color{blue}{ \textbf{ 0.774 } } & 0.854 & 0.871 \\
\cline{2-12}
 & \multirow{4}{*}{95} & 1  & 0.478 & 0.866 & 0.908 & 0.055 & 0.039 & 0.037 & 0.618 & 0.852 & 0.870 \\
 &  & 2 & 0.712 & 0.896 & 0.922 & 0.076 & 0.069 & 0.075 & 0.703 & 0.792 & 0.790 \\
 &  & 3 & 0.758 & 0.888 & 0.918 & 0.098 & 0.096 & 0.098 & 0.681 & 0.736 & 0.743 \\
 &  & 5 & \color{blue}{ \textbf{ 0.808 } } & \color{blue}{ \textbf{ 0.906 } } & 0.922 & 0.108 & 0.115 & 0.116 & 0.685 & 0.709 & 0.713 \\
\cline{2-12}
\bhline{1pt}
\multirow{12}{*}{2} & \multirow{4}{*}{99.9} & 1  & 0.109 & 0.328 & 0.813 & \color{blue}{ \textbf{ 0.004 } } & 0.025 & 0.005 & 0.319 & 0.606 & 0.929 \\
 &  & 2 & 0.173 & 0.458 & 0.853 & 0.007 & 0.026 & 0.005 & 0.436 & 0.714 & 0.944 \\
 &  & 3 & 0.216 & 0.515 & 0.863 & 0.009 & 0.027 & 0.007 & 0.500 & 0.750 & 0.945 \\
 &  & 5 & 0.250 & 0.608 & 0.884 & 0.018 & 0.033 & \color{blue}{ \textbf{ 0.005 } } & 0.531 & 0.795 & \color{blue}{ \textbf{ 0.956 } } \\
\cline{2-12}
 & \multirow{4}{*}{99} & 1  & 0.246 & 0.603 & 0.854 & 0.022 & \color{blue}{ \textbf{ 0.022 } } & 0.009 & 0.521 & 0.810 & 0.939 \\
 &  & 2 & 0.376 & 0.721 & 0.875 & 0.030 & 0.028 & 0.010 & 0.643 & 0.857 & 0.944 \\
 &  & 3 & 0.439 & 0.774 & 0.880 & 0.045 & 0.027 & 0.015 & 0.672 & \color{blue}{ \textbf{ 0.881 } } & 0.938 \\
 &  & 5 & 0.537 & 0.802 & \color{blue}{ \textbf{ 0.896 } } & 0.060 & 0.038 & 0.020 & 0.716 & 0.877 & 0.937 \\
\cline{2-12}
 & \multirow{4}{*}{95} & 1  & 0.417 & 0.775 & 0.865 & 0.055 & 0.032 & 0.023 & 0.641 & 0.874 & 0.921 \\
 &  & 2 & 0.581 & 0.815 & 0.875 & 0.071 & 0.056 & 0.045 & 0.727 & 0.856 & 0.894 \\
 &  & 3 & 0.627 & 0.841 & 0.885 & 0.090 & 0.067 & 0.066 & \color{blue}{ \textbf{ 0.729 } } & 0.852 & 0.870 \\
 &  & 5 & \color{blue}{ \textbf{ 0.664 } } & \color{blue}{ \textbf{ 0.850 } } & 0.886 & 0.106 & 0.093 & 0.090 & 0.728 & 0.823 & 0.840 \\
\cline{2-12}
\bhline{1pt}
\multirow{12}{*}{3} & \multirow{4}{*}{99.9} & 1  & 0.161 & 0.277 & 0.620 & \color{blue}{ \textbf{ 0.010 } } & \color{blue}{ \textbf{ 0.050 } } & \color{blue}{ \textbf{ 0.043 } } & 0.414 & 0.537 & 0.800 \\
 &  & 2 & 0.197 & 0.338 & 0.659 & 0.018 & 0.060 & 0.046 & 0.463 & 0.589 & 0.817 \\
 &  & 3 & 0.214 & 0.379 & 0.685 & 0.024 & 0.063 & 0.046 & 0.482 & 0.622 & 0.830 \\
 &  & 5 & 0.257 & 0.410 & 0.713 & 0.030 & 0.078 & 0.048 & 0.531 & 0.634 & 0.841 \\
\cline{2-12}
 & \multirow{4}{*}{99} & 1  & 0.233 & 0.428 & 0.675 & 0.032 & 0.055 & 0.047 & 0.501 & 0.669 & 0.824 \\
 &  & 2 & 0.297 & 0.513 & 0.721 & 0.054 & 0.066 & 0.045 & 0.554 & 0.717 & 0.848 \\
 &  & 3 & 0.330 & 0.566 & 0.732 & 0.068 & 0.068 & 0.049 & 0.574 & 0.747 & \color{blue}{ \textbf{ 0.849 } } \\
 &  & 5 & 0.399 & 0.597 & \color{blue}{ \textbf{ 0.739 } } & 0.078 & 0.075 & 0.052 & 0.626 & 0.758 & 0.849 \\
\cline{2-12}
 & \multirow{4}{*}{95} & 1  & 0.337 & 0.572 & 0.689 & 0.071 & 0.067 & 0.050 & 0.579 & 0.751 & 0.828 \\
 &  & 2 & 0.409 & 0.625 & 0.725 & 0.101 & 0.074 & 0.052 & 0.614 & 0.774 & 0.843 \\
 &  & 3 & 0.459 & \color{blue}{ \textbf{ 0.672 } } & 0.734 & 0.107 & 0.074 & 0.061 & 0.645 & \color{blue}{ \textbf{ 0.797 } } & 0.838 \\
 &  & 5 & \color{blue}{ \textbf{ 0.497 } } & 0.671 & 0.729 & 0.112 & 0.089 & 0.070 & \color{blue}{ \textbf{ 0.666 } } & 0.783 & 0.828 \\
\cline{2-12}
\bhline{1pt}
\end{tabular}
\label{tab:MoonResult}
\end{table}

\begin{figure}[htb!]
    \centering
    \includegraphics[width=0.95\linewidth]{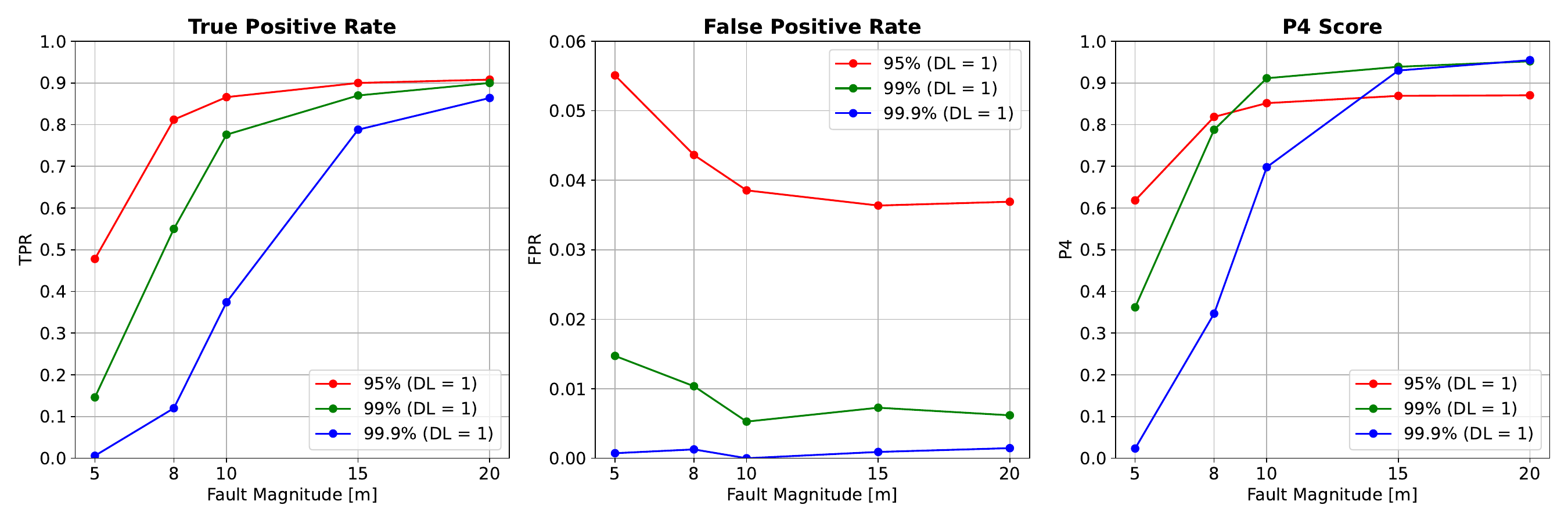}
    \caption{The detection results for different fault magnitudes and the test statistic threshold $\teststatisticthreshold$. The number of faults and the detection length is fixed to 1. Lowering the test statistic threshold results in larger TPR and lower FPR. The $P_4$ score is larger at smaller thresholds for smaller fault magnitudes, and at larger thresholds for larger fault magnitudes. }
    \label{fig:comparison_thrshold}
\end{figure}

\begin{figure}[htb!]
    \centering
    \includegraphics[width=0.95\linewidth]{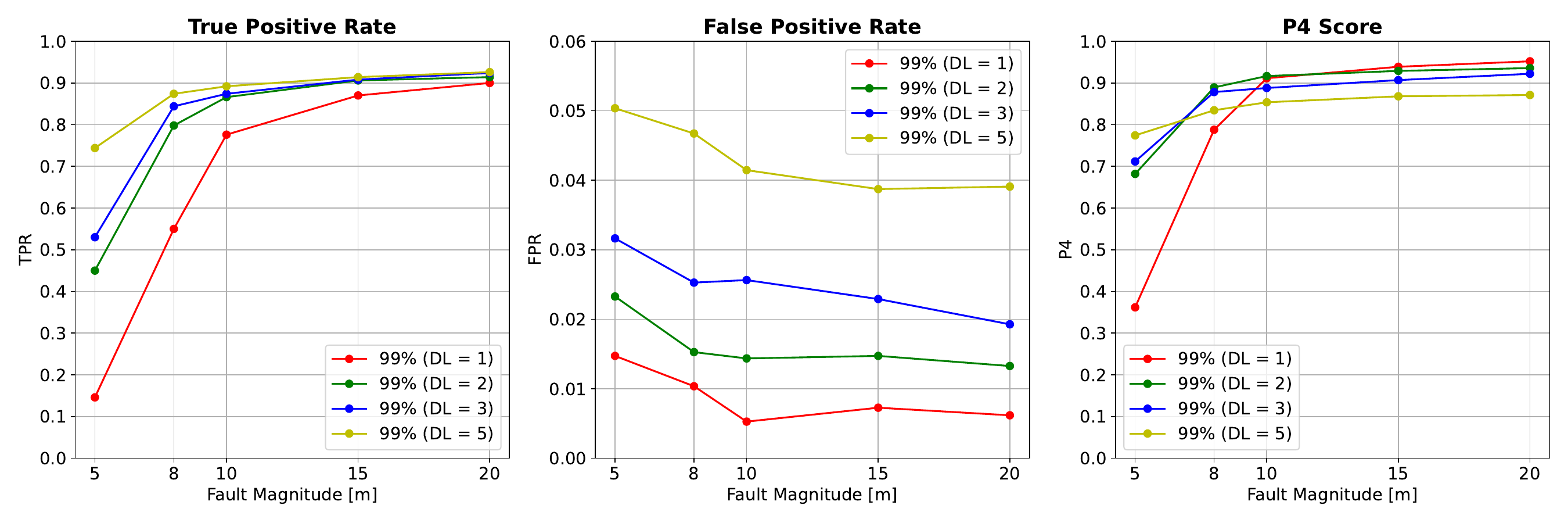}
    \caption{The detection results for different fault magnitudes and the detection length $DL$. The test statistic threshold is fixed to $\teststatisticthreshold^{99}$, and the number of faults is fixed to 1. Increasing the $DL$ results in a larger TPR and lower FPR. The $P_4$ score is larger at longer $DL$ for smaller fault magnitudes, and at shorter $DL$ for larger fault magnitudes.}
    \label{fig:comparison_dl}
\end{figure}

\begin{figure}[htb!]
    \centering
    \includegraphics[width=0.95\linewidth]{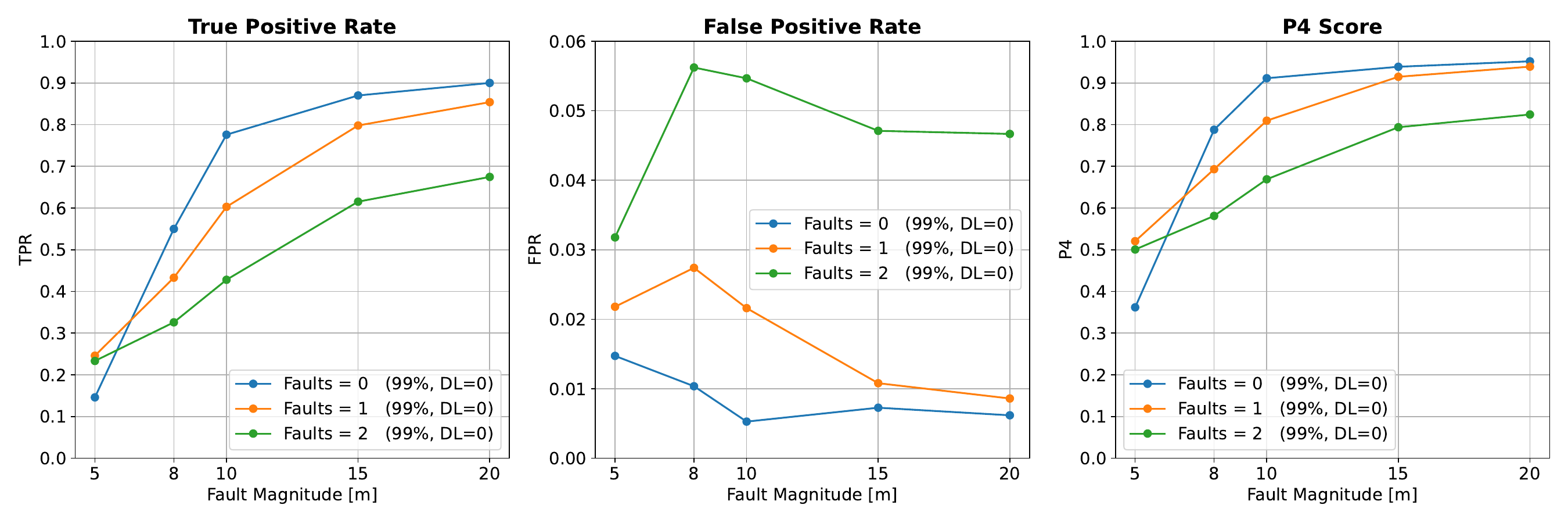}
    \caption{The detection results for different fault magnitudes and number of fault satellites. The test statistic threshold is fixed to $\teststatisticthreshold^{99}$, and the detection length is fixed to 1. In general, increasing the number of faults results in lower TPR and higher FPR. }
    \label{fig:comparison_nfaults}
\end{figure}


\section{Conclusion}
\label{sec:conclusion}
In this paper, we introduced a fault detection framework for autonomous satellite constellations using inter-satellite range (ISR) measurements. 
Our fault detection approach is based solely on the ISR measurements, and does not require the availability of precise ephemeris information. Compared to traditional fault monitoring methods that require ground station monitoring or precise ephemeris, our proposed method is suited for lunar environments where dedicated monitoring stations on the planetary surface would not be available.
 
In the paper, we start by providing mathematical proofs of the sufficient and necessary conditions of the graph topology required for fault detection. Next, we introduce several properties of the rank of the EDM and GCEDM. Based on these properties, we demonstrated that satellite fault can be detected by analyzing the singular values and the singular vectors of the geometric centered Euclidean Distance Matrix (GCEDM) derived from ISR measurements, using vertex-redundantly rigid graphs.
By using multiple vertex-redundantly rigid subgraphs found by a clique-finding algorithm to detect faults , the proposed fault detection framework can be used for dynamic geometries and graph topologies, enhancing the robustness and reliability of fault detection.
Through simulations of an ELFO constellation around the Moon, we demonstrated the effectiveness of our fault detection framework and organized the relationship between the environmental parameters, hyperparameters, and detection performance.

This research marks a step towards establishing reliable operation of satellite constellations in extraterrestrial environments, which is the key architecture to support future exploration missions on the Moon.
Future work will focus on developing a model that can predict the tails of the non-fault test statistics distribution from the ranges, selecting effective subgraphs to reduce computation, looking into distributed implementation, and combining the proposed algorithms with residual-based fault detection methods to further enhance the robustness of the algorithm.

\section*{acknowledgements}
We acknowledge Derek Knowles for reviewing the paper.
This material is based upon work supported by The Nakajima Foundation and the National Science Foundation under Grant No. DGE-1656518 and CNS-2006162. 

\bibliographystyle{apalike}
\bibliography{ref}

\appendix

\section{Proof of Proposition 3.1 }
\label{sec:app1}

\begin{proof} 
    As discussed in \cite{dokmanic2015euclidean}, the noiseless- and faultess-EDM, $\EDM^{n, d, 0}_{ij} = \| \singlepoint_i - \singlepoint_j\|^2$, based on a collection of points $\PointMat \in \R^{d \times n}$ can be constructed with Equation~\eqref{eq:edmnoisless}.
    \begin{equation}
        \EDM^{n, d, 0} = \onesvec \, \diag(\PointMat^\top \PointMat)^\top - 2 \PointMat^\top \PointMat + \diag(\PointMat^\top \PointMat) \onesvec^\top \label{eq:edmnoisless}
    \end{equation}
    where $\diag(\PointMat^\top \PointMat) \in \R^n$ is the vector formed from the diagonal entries of $\PointMat^\top \PointMat$.
    The first and last matrices of Equation~\eqref{eq:edmnoisless} are rank 1 matrices by construction. 
    These rank 1 matrix each contribute to the matrix rank except for degenerate cases where the points are equally spread apart ($\PointMat \onesvec = \zeromat$). 
    The middle term includes the Gram matrix $\PointMat^\top \PointMat$, which is rank $d$ except for degenerate cases where the points lie on a lower dimensional hyperplane (i.e., $\matrank(\PointMat^\top \PointMat) = 2$ if the points lie on a plane).
    So, in line with \cite{dokmanic2015euclidean}, $\matrank(\EDM^{n, d, 0}) \leq d + 2$, where the condition holds with equality outside of the aforementioned degenerate cases. 

    During a fault, we encounter cross terms with the bias following Equation~\ref{eq:edmwithfault}.
    \begin{equation}
        \begin{aligned}
            \EDM^{n, d, m}_{ij} &= (\| \singlepoint_i - \singlepoint_j\| + \fault_{ij})^2 =  \| \singlepoint_i - \singlepoint_j\|^2 + 2\| \singlepoint_i - \singlepoint_j\| \fault_{ij} + \fault_{ij}^2 \\
            &= \| \singlepoint_i - \singlepoint_j\|^2 + \fault_{ij} \left( 2\| \singlepoint_i - \singlepoint_j\| + \fault_{ij} \right) \\
            &= \| \singlepoint_i - \singlepoint_j\|^2 + \fault_{ij} \left(\sqrtterm_{ij} + \fault_{ij} \right)
            \label{eq:edmwithfault}
        \end{aligned}
    \end{equation}
    Where $\sqrtterm_{ij}$ is two times the elementwise square root of the corresponding entry in the EDM. 
    As a matrix, we can write $\SqrtMatSuper_{ij} = s_{ij}$.
    We can write the bias term $\FaultMatSuper_{ij} = b_{ij}$ as a matrix as well.
    Both $\SqrtMatSuper$ and $\FaultMatSuper$ are real, symmetric matrices.  
    Then, we arrive at Equation~\eqref{eq:faultaddedmatrix}.
    \begin{equation}
        \EDM^{n, d, m} = \EDM^{n, d, 0} + \FaultMatSuper \hadamard (\SqrtMatSuper + \FaultMatSuper)
        \label{eq:faultaddedmatrix}
    \end{equation}
    where $\hadamard$ is the element-wise or Hadamard product. 
    For the sake of illustration, consider the case with three faults realized at satellites $k=\{2, 3, 4\}$ for $n > 7$, without loss of generality.
    \begin{equation}
        \FaultMat^{n, \{2, 3, 4\}} = \faultmag \begin{bmatrix}
            0 & 1 & 1 & 1 & 0 & \cdots & 0 \\
            1 & 0 & 2 & 2 & 1 & \cdots & 1 \\
            1 & 2 & 0 & 2 & 1 & \cdots & 1 \\
            1 & 2 & 2 & 0 & 1 & \cdots & 1 \\
            0 & 1 & 1 & 1 & 0 & \cdots & 0 \\
            \vdots & \vdots & \vdots & \vdots & \vdots & \ddots & \vdots \\
            0 & 1 & 1 & 1 & 0 & \cdots & 0
        \end{bmatrix} = \faultmag \begin{bmatrix}
            0 & 1 & 1 & 1 \\
            1 & 0 & 2 & 2 \\
            1 & 2 & 0 & 2 \\
            1 & 2 & 2 & 0 \\
            0 & 1 & 1 & 1 \\
            \vdots & \vdots & \vdots & \vdots \\
            0 & 1 & 1 & 1\\
        \end{bmatrix} \begin{bmatrix}
            1 & 0 & 0 & 0 & 1 & \cdots & 1 \\
            0 & 1 & 0 & 0 & 0 & \cdots & 0 \\
            0 & 0 & 1 & 0 & 0 & \cdots & 0 \\
            0 & 0 & 0 & 1 & 0 & \cdots & 0 \\
        \end{bmatrix}
        \label{eq:faultmatThree}
    \end{equation}
    In the rank-decomposition of $\FaultMatSuper$, the column matrix has the first column has a one in each fault entry and a zero otherwise.
    This column spans all the $n - m$ fault-free columns. 
    The remaining $m$ columns correspond to the columns in $\FaultMatSuper$ with faults and have a zero at the fault entry. 
    $\FaultMatSuper$ achieves full rank when the number of faults reaches $m = n-1$.  
    In general, $\matrank(\FaultMatSuper) = \min(1 + m, n)$ if $m > 0$ and $\matrank(\FaultMat^{n, 0}) = 0$.
    By similar argument, $\matrank(\FaultMatSuper \hadamard \FaultMatSuper) = \matrank(\FaultMatSuper)$ since we will have one column indicating the fault entries and $m$ columns associated with the columns of $\FaultMatSuper \hadamard \FaultMatSuper$ with faults, which are the same columns as $\FaultMatSuper$.
    This overall rank relationship often holds even if the faults are of different non-zero magnitudes.
    However, if the faults cancel out, the rank will drop.
    For example, in the case above, if $b_2 = 2 \faultmag$ and $b_3 = b_4 = - \faultmag$, then $\matrank(\FaultMatSuper) = 3$.
    Therefore, $\matrank(\FaultMatSuper) \leq \min(1 + m, n)$ for arbitrary fault sizes.
    Nevertheless, the term $\matrank(\FaultMatSuper \hadamard \FaultMatSuper)$ can still be rank $m + 1$ even if $\matrank(\FaultMatSuper) < m + 1$ since the squaring can remove the linear cancellation.
    
    However, the span of the columns of $\FaultMatSuper \hadamard \SqrtMatSuper$ will generally span the columns $\FaultMatSuper \hadamard \FaultMatSuper$. 
    Again, for the sake of illustration, we extend the example above with faults at satellites $k=\{2, 3, 4\}$ for $n > 7$, without loss of generality.
    \begin{equation}
    \begin{aligned}
        \FaultMat^{n, \{2, 3, 4\}} \hadamard \SqrtMatSuper &= \faultmag \begin{bmatrix}
            0      &   s_{12} &   s_{13} &   s_{14} & 0      & \cdots & 0 \\
            s_{12} & 0        & 2 s_{23} & 2 s_{24} & s_{25} & \cdots & s_{2n} \\
            s_{13} & 2 s_{23} & 0        & 2 s_{34} & s_{35} & \cdots & s_{3n} \\
            s_{14} & 2 s_{24} & 2 s_{34} & 0        & s_{45} & \cdots & s_{4n} \\
            0      &   s_{25} &   s_{35} &   s_{45} & 0      & \cdots & 0 \\
            \vdots & \vdots & \vdots & \vdots & \vdots & \ddots & \vdots \\
            0      &   s_{2n} &   s_{3n} &   s_{4n} & 0 & \cdots & 0
        \end{bmatrix} \\
        &= \faultmag \begin{bmatrix}
            0      &   s_{12} &   s_{13} &   s_{14} & 0      & 0 \\
            s_{12} & 0        & 2 s_{23} & 2 s_{24} & s_{25} & s_{26} \\
            s_{13} & 2 s_{23} & 0        & 2 s_{34} & s_{35} & s_{36} \\
            s_{14} & 2 s_{24} & 2 s_{34} & 0        & s_{45} & s_{46} \\
            0      &   s_{25} &   s_{35} &   s_{45} & 0      & 0 \\
            \vdots &   \vdots &   \vdots &   \vdots & \vdots & \vdots \\
            0      &   s_{2n} &   s_{3n} &   s_{4n} & 0      & 0 \\
        \end{bmatrix} \begin{bmatrix}
            1 & 0 & 0 & 0 & 0 & 0 & \rrefleftover_{17} & \cdots & \rrefleftover_{1n} \\
            0 & 1 & 0 & 0 & 0 & 0 & 0 & \cdots & 0 \\
            0 & 0 & 1 & 0 & 0 & 0 & 0 & \cdots & 0 \\
            0 & 0 & 0 & 1 & 0 & 0 & 0 & \cdots & 0 \\
            0 & 0 & 0 & 0 & 1 & 0 & \rrefleftover_{57} & \cdots & \rrefleftover_{5n} \\
            0 & 0 & 0 & 0 & 0 & 1 & \rrefleftover_{67} & \cdots & \rrefleftover_{6n} \\
        \end{bmatrix}
    \end{aligned}
        \label{eq:faultmatThreeLeftOver}
    \end{equation}
    where $\rrefleftover_{ij}$ are leftover terms from the row-reduction. 
    In the rank-decomposition of $\FaultMatSuper \hadamard \SqrtMatSuper$, $m$ columns correspond to the columns in $\FaultMatSuper$ with faults and have a zero at the fault entry, just as before.
    However, now, one column is generally not enough to span the fault-free columns since the entries in $\SqrtMatSuper$ are not linearly related.
    For example, $(s_{12}, s_{13}, s_{14})$ is generally not a linear scaling of $(s_{25}, s_{35}, s_{45})$.
    These fault-free columns constitute an $m$ dimensional subspace, from which will need $m$ columns to fully span.
    $\FaultMatSuper \hadamard \SqrtMatSuper$ achieves full rank when the number of faults reaches $m = \lceil n/2 \rceil$, where $\lceil \cdot \rceil$ is the ceiling function.
    At that point, the subspace of fault-free columns is large enough that the fault columns will contribute to the span.
    Therefore, $\matrank(\FaultMatSuper \hadamard \SqrtMatSuper) \leq \min(2m, n)$ where we do not have equality when $\SqrtMatSuper$ is degenerate (i.e., if the fault satellites are above or below a plane of fault-free satellites). 

    The columns needed to span $\FaultMatSuper \hadamard \SqrtMatSuper$ are the same as those needed for $\FaultMatSuper \hadamard \FaultMatSuper$, with the same sparsity pattern.
    Therefore, we will have
    \begin{equation}
        \matrank(\FaultMatSuper \hadamard (\SqrtMatSuper + \FaultMatSuper)) = \matrank(\FaultMatSuper \hadamard \SqrtMatSuper)  \leq \min(2m, n)
    \end{equation}
    Lastly, the subspace spanned with the fault contributions is distinct from the subspace spanned by the EDM.
    Therefore, 
    \begin{equation}
        \matrank(\EDM^{n, d, m}) \leq \min(d + 2 + 2m, n)
    \end{equation}
    where the equality holds in non-degenerate cases.

\end{proof}

\section{Proof of Proposition 3.2}
\label{sec:app2}

\begin{proof}
    First, by rank inequality $\matrank(AB) \leq \min(\matrank(A), \matrank(B))$ for arbitrary matrices $A$ and $B$. 
    The rank of the geometric centring matrix is $\matrank(J^n) = n - 1$.
    So, $\matrank(\GCEDM^{n,d,k}) \leq n - 1$.
    However, this bound is too loose when there are few faults.
    Expanding the expression for geometric centering yields Equation~\eqref{eq:geocenteringexpanded}.
    \begin{equation}
        \begin{aligned}
            \GCEDM^{n,d,m} &= -\frac{1}{2} J^n \EDM^{n, d, m} J^n \\
            &= -\frac{1}{2} J^n (\EDM^{n, d, 0} + \FaultMatSuper \hadamard (\SqrtMatSuper + \FaultMatSuper)) J^n \\
            &= -\frac{1}{2} J^n (\onesvec \, \diag(\PointMat^\top \PointMat)^\top - 2 \PointMat^\top \PointMat + \diag(\PointMat^\top \PointMat) \onesvec^\top) J^n - \frac{1}{2} J^n (\FaultMatSuper \hadamard (\SqrtMatSuper + \FaultMatSuper)) J^n
        \end{aligned} \label{eq:geocenteringexpanded}
    \end{equation}

    First, geometric centering removes the two rank 1 matrices, as shown in Equations~\eqref{eq:geocenterterm1} and \eqref{eq:geocenterterm3} \citep{dokmanic2015euclidean}.
    \begin{align}
        -\frac{1}{2} J^n (\onesvec \, \diag(\PointMat^\top \PointMat)^\top ) J^n &= -\frac{1}{2} (\onesvec \, \diag(\PointMat^\top \PointMat)^\top  - \onesvec \, \diag(\PointMat^\top \PointMat)^\top ) (\eye^n - \frac{1}{n} \onesvec {\onesvec}^{\top}) = 0 \label{eq:geocenterterm1} \\
        -\frac{1}{2} J^n (\diag(\PointMat^\top \PointMat) \onesvec^\top) J^n &= -\frac{1}{2} (\eye^n - \frac{1}{n} \onesvec {\onesvec}^{\top}) (\diag(\PointMat^\top \PointMat) \onesvec^\top  - \diag(\PointMat^\top \PointMat) \onesvec^\top) = 0 \label{eq:geocenterterm3}
    \end{align}
    For the $-2 \PointMat^\top \PointMat$ term, notice that the mean of the points is $\meanPoints = \frac{1}{n} X \onesvec \in \R^d$. Using this property yields Equation~\eqref{eq:geocenterterm2xTx}, which is a Gram matrix of the form $\PointMat_c^\top \PointMat_c$ for the point matrix $\PointMat_c$ centered about the origin \citep{dokmanic2015euclidean}.
    This shift will not change the rank, meaning $\matrank(\PointMat^\top \PointMat) = \matrank(\PointMat_c^\top \PointMat_c)$.
    \begin{equation}
        -\frac{1}{2} J^n (-2 \PointMat^\top \PointMat) J^n = (\PointMat - \meanPoints {\onesvec}^{\top})^\top (\PointMat - \meanPoints {\onesvec}^{\top}) = \PointMat_c^\top \PointMat_c \label{eq:geocenterterm2xTx}
    \end{equation}
    The remaining term is $\frac{1}{2} J^n (\FaultMatSuper \hadamard (\SqrtMatSuper + \FaultMatSuper)) J^n$. 
    By rank inequality, 
    \begin{equation}
        \matrank(J^n (\FaultMatSuper \hadamard (\SqrtMatSuper + \FaultMatSuper)) J^n) \leq \min(\matrank(\FaultMatSuper \hadamard (\SqrtMatSuper + \FaultMatSuper)), \matrank(J^n)) = \min(2m, n - 1)
    \end{equation}
    So, we are left with
    \begin{equation}
        \matrank(\GCEDM^{n,d,k}) \leq \min(d + 2m, n - 1) 
    \end{equation}

\end{proof}
\section{Proof of Proposition 3.3}
\label{sec:app3}
\begin{proof}
    In terms of rank, the noise acts as many small faults on each satellite, with the same sparsity structure as $\FaultMatSuper$ in Proposition~\ref{prop:edmrank}, almost surely with $m = n$ since there is zero probability mass that the randomly sampled noise is exactly zero or exactly cancel out.
    Using Proposition~\ref{prop:geocenteredm}, 
    \begin{equation}
        \matrank(\GCEDMnoisy^{n,d,m}) = \matrank(\GCEDM^{n,d,n}) \leq \min(d + 2n, n - 1) = n - 1
    \end{equation}
    Therefore, succinctly, $\matrank(\GCEDMnoisy^{n,d,m}) = n - 1$, almost surely.
\end{proof}

\end{document}